%% file: iclr2026_conference.tex
\documentclass{article} 
\usepackage{iclr2026_conference,times}

\input{math_commands.tex}

\usepackage{hyperref}
\usepackage{url}

\usepackage{booktabs}       
\usepackage{amsfonts}       
\usepackage{nicefrac}       
\usepackage{microtype}      
\usepackage[table]{xcolor}         
\usepackage{amsmath}
\usepackage{amsthm}
\usepackage{amssymb}

\usepackage{dirtytalk}
\usepackage{cleveref}
\usepackage{float}
\usepackage{algorithm}
\usepackage{algpseudocode}
\usepackage{makecell, cellspace, colortbl, multirow}
\usepackage{mismath} 
\usepackage{footnote}
\usepackage{bbm}
\usepackage{thm-restate}

\usepackage{bm}

\crefname{assumption}{assumption}{assumptions}
\creflabelformat{assumption}{#2#1#3}

\newtheorem{lem}{Lemma}
\newtheorem{definition}{Definition}
\newtheorem{thm}{Theorem}
\newtheorem{assumption}{Assumption}

\DeclareMathOperator{\Diam}{Diam}

\theoremstyle{theorem}
\newtheorem{innercustomthm}{Theorem}
\newenvironment{restate-thm}[1]
{\renewcommand\theinnercustomthm{#1}\innercustomthm}
{\endinnercustomthm}

\newtheorem{innercustomlemma}{Lemma}
\newenvironment{restate-lem}[1]
{\renewcommand\theinnercustomlemma{#1}\innercustomlemma}
{\endinnercustomlemma}

\newtheorem{innercustomproposition}{Proposition}
\newenvironment{restate-proposition}[1]
{\renewcommand\theinnercustomproposition{#1}\innercustomproposition}
{\endinnercustomproposition}

\theoremstyle{definition}
\newtheorem{innercustomdef}{Definition}
\newenvironment{restate-definition}[1]
{\renewcommand\theinnercustomdef{#1}\innercustomdef}
{\endinnercustomdef}

\title{Statistical Guarantees for Offline Domain Randomization}

\author{%
  Arnaud Fickinger\thanks{Equal  contribution.} \kern.2em \textsuperscript{1} \\
  UC Berkeley \\
  \And
  Abderrahim Bendahi\footnotemark[1]\kern.4em\thanks{Work done during internship at UC Berkeley.} \kern.2em \textsuperscript{2}\\
  École Polytechnique \\
  \And
  Stuart Russell\textsuperscript{3} \\
  UC Berkeley \\
}

\iclrfinalcopy
\begin{document}

\maketitle

\begingroup
\renewcommand{\thefootnote}{\arabic{footnote}}
\footnotetext[1]{\texttt{arnaud.fickinger@berkeley.edu}}
\footnotetext[2]{\texttt{abderrahim.bendahi@polytechnique.edu}}
\footnotetext[3]{\texttt{russell@berkeley.edu}}
\endgroup

\begin{abstract}
Reinforcement-learning (RL) agents often struggle when deployed from simulation to the real-world. A dominant strategy for reducing the sim-to-real gap is domain randomization (DR) which trains the policy across many simulators produced by sampling dynamics parameters, but standard DR ignores offline data already available from the real system. We study offline domain randomization (ODR), which first fits a distribution over simulator parameters to an offline dataset. While a growing body of empirical work reports substantial gains with algorithms such as DROPO \citep{tiboni2023droposimtorealtransferoffline}, the theoretical foundations of ODR remain largely unexplored. In this work, we cast ODR as a maximum-likelihood estimation over a parametric simulator family and provide statistical guarantees: under mild regularity and identifiability conditions, the estimator is weakly consistent (it converges in probability to the true dynamics as data grows), and it becomes strongly consistent (i.e., it converges almost surely to the true dynamics) when an additional uniform Lipschitz continuity assumption holds. We examine the practicality of these assumptions and outline relaxations that justify ODR’s applicability across a broader range of settings. Taken together, our results place ODR on a principled footing and clarify when offline data can soundly guide the choice of a randomization distribution for downstream offline RL.
\end{abstract}

\section{Introduction}
In recent years, RL has achieved many empirical successes, attaining human-level performance in tasks such as games \citep{DBLP:journals/corr/MnihKSGAWR13, go}, robotics \citep{DBLP:journals/corr/abs-1806-10293, DBLP:journals/corr/SchulmanLMJA15}, and recommender systems \citep{recommender1, recommender2}. Yet, RL algorithms often require vast amounts of training data to learn effective policies, which severely limits their applicability in real world settings where data collection is expensive, time‐consuming, or unsafe \citep{DBLP:journals/corr/abs-2005-01643, driving}.

\textit{Sim-to-real transfer} tackles this problem by learning in simulation and transferring the resulting policy to the real world \citep{DBLP:journals/corr/SadeghiL16, DBLP:journals/corr/abs-1804-10332, DBLP:journals/corr/abs-2009-13303}. However, although simulation provides fast and safe data collection, inevitable discrepancies between the simulated dynamics and the real world, commonly termed the \textit{sim-to-real gap}, typically induce a drop in performance upon deployment.

One of the most widely-used approaches to bridge this gap is \textit{domain randomization} (DR). Rather than training on a single fixed simulator, DR defines a family of simulators parameterized by physical factors (e.g., masses, friction coefficients, sensor noise) and at the start of each episode \textit{randomly samples} one instance from this family for training. DR has enabled zero-shot transfer in robotic control \citep{DBLP:journals/corr/TobinFRSZA17, DBLP:journals/corr/SadeghiL16}, dexterous manipulation \citep{DBLP:journals/corr/abs-1808-00177} and agile locomotion \citep{DBLP:journals/corr/abs-1710-06537}. 

Despite this empirical track record, the choice of \textit{how} to randomize is a fundamental challenge. In the original form of DR \citep{DBLP:journals/corr/TobinFRSZA17, DBLP:journals/corr/SadeghiL16}, broad \textit{uniform} ranges that look reasonable for every parameter are chosen. While recent theoretical work \citep{chen2022understandingdomainrandomizationsimtoreal} shows that such \textit{uniform DR} (UDR) can indeed bound the sim-to-real gap, the bound unfavorably scales in $O\left( N^3 \log(N) \right)$ with respect to the number of candidate simulators, in part because UDR ignores any data already available from the target system.

In contrast, \textit{Offline Domain Randomization} exploits a static dataset from the real environment before policy training to fit a sampling distribution that concentrates on plausible dynamics while remaining stochastic. Empirically, ODR variants such as DROID \citep{DROID} or DROPO \citep{tiboni2023droposimtorealtransferoffline} recover parameter distributions that explain the data and yield stronger zero-shot transfer than hand-tuned UDR. Yet, to the best of our knowledge, ODR lacks a principled foundation: we do not know (i) whether the fitted distribution converges to the real dynamics as data grows, nor (ii) how much it actually reduces the sim-to-real gap compared with UDR.

\paragraph{Our Contributions:}
\begin{itemize}
  \item \textbf{Weak consistency (Section~\ref{sec:main-results}).} We formalize ODR as maximum-likelihood estimation over a parametric simulator family and prove \emph{weak consistency}: under mild regularity, positivity, and identifiability assumptions, empirical maximizers converge in probability to the population maximizers. 
  \item \textbf{Strong consistency (Section~\ref{sec:strong-consistency}).} Adding a single \emph{uniform Lipschitz continuity} assumption on the likelihood, we upgrade convergence to \emph{strong consistency}: the ODR estimator converges almost surely to the true parameter when it is uniquely identified.
  \item \textbf{Assumptions in practice: discussion and relaxations (Section~\ref{sec:assumptions-relaxations}).} We analyze when the assumptions hold and provide drop-in relaxations and diagnostics: replacing i.i.d.\ by strict stationarity and ergodicity for the, weakening mixture positivity via a logarithmic tail condition, and giving simple sufficient conditions that imply the uniform Lipschitz requirement. 
\end{itemize}

\section{Related Works}
\paragraph{Sim-to-real transfer} The \textit{sim-to-real gap} has led to extensive research in sim-to-real transfer. Early works exploited system identification or progressive networks to adapt controllers online \citep{evolutionary-robotics, kober}, while more recent efforts have focused on purely offline training in high-fidelity simulators. Although zero‑shot transfer has been demonstrated for specific settings such as legged locomotion \citep{DBLP:journals/corr/abs-1710-06537}, dexterous manipulation \citep{chebotar, DBLP:journals/corr/abs-1808-00177} and visuomotor control \citep{Rusu2016SimtoRealRL} a noticeable performance gap persists in unstructured environments. Similar ideas have been explored in autonomous driving \citep{9025396, autonomousDR}.

\paragraph{Domain randomization} Domain randomization (DR) varies environment parameters at every training episode with the goal of producing policies that generalize across the induced simulator family. Vision-based DR first showed zero-shot transfer for quadrotor flight from purely synthetic images \citep{DBLP:journals/corr/SadeghiL16}, and dynamics randomization extended this success to legged robots and manipulation \citep{DBLP:journals/corr/abs-1808-00177}. To avoid manual tuning of randomization ranges, online methods adapt the DR distribution using real-world feedback. Ensemble-based robust optimization and Bayesian optimization techniques refine parameters via real rollouts \citep{EPOpt, muratore}, while meta RL further accelerates adaptation \citep{meta-rl, meta-rl-2019}. However, these require repeated---and potentially unsafe---hardware interactions during training.

\paragraph{Offline domain randomization} A growing line of work aims to find the best strategy to perform domain randomization from a fixed offline dataset, obviating any further real-world trials. DROID \citep{DROID} tunes simulator parameters using CMA-ES \citep{CMA-ES, CMA-ES-2} with the $L^2$ distance between a single human demonstration and its simulated counterpart as objective function. BayesSim \citep{BayesSim} trains a conditional density estimator to predict a posterior over simulator parameters given offline off-policy rollouts. Most recently, DROPO \citep{tiboni2023droposimtorealtransferoffline} introduces a likelihood-based framework that fits both the mean and covariance of a Gaussian parameter distribution by maximizing the log-likelihood of the offline data under a mixture simulator. This approach recovers rich uncertainty estimates, handles non-differentiable black-box simulators via gradient-free optimizers, and outperforms DROID, BayesSim and uniform DR in zero-shot transfer on standard benchmarks without any on-policy real-world interaction.

\paragraph{Theoretical analyses} Let $M$ be the number of candidate simulators and $H$ the horizon length. \cite{chen2022understandingdomainrandomizationsimtoreal} modeled uniform DR as a \textit{latent MDP} and proved that the performance gap between the optimal policy in the true system and the policy trained with DR scales as $O(M^3 \log(MH))$\footnote{The original paper derived a looser bound, see \Cref{app:improvement} for a tighter derivation.} in the case where the simulator class is finite and separated and $O(\sqrt{M^3 H \log(MH)})$ in the finite non-separated simulator class case. Other works have studied the information-theoretical limit of sim-to-real transfer \citep{jiang}, PAC-style guarantees via approximate simulators \citep{feng} and generalization in rich-observation MDPs \citep{zhong, krishnamurthy2016pac}. But none address the statistical benefits of offline DR. Our work bridges this gap by providing the first consistency proofs and finite-sample gap bounds for offline DR, thereby unifying empirical successes and theoretical understanding in a single framework.

\section{Problem Setup and ODR Formulation}\label{sec:setting}

\paragraph{Episodic MDPs}
We consider the episodic RL setting where each MDP corresponds to $\mathcal{M} = (\mathcal{S}, \mathcal{A}, P_{\mathcal M}, R, H, s_1)$. $\mathcal{S}$ is the set of states, $\mathcal{A}$ is the set of actions, $P_\mathcal{M}\colon \mathcal{S} \times \mathcal{A} \xrightarrow{} \Delta(\mathcal{A})$ is the transition probability matrix, $R \colon \mathcal{S} \times \mathcal{A} \xrightarrow{} [0, 1]$ is the reward function, $H$ is the number of steps of each episode, and $s_1$ is the initial state at step $h = 1$; we assume w.l.o.g.~that the agent starts from the same state in each episode.

At step $h \in [H]$, the agent observes the current state $s_h \in \mathcal{S}$, takes action $a_h \in \mathcal{A}$, receives reward $R(s_h, a_h)$, and moves to state $s_{h+1}$ with probability $P_{\mathcal M}( s_{h + 1} \mid s_h, a_h)$. The episode ends when state $s_{H+1}$ is reached.

A policy \(\pi\) is a sequence \(\{\pi_h\}_{h=1}^H\) where each \(\pi_h\) maps histories \(\mathrm{traj}_h=\{(s_1,a_1,\dots,s_h)\}\) to action distributions. Denote by \(\Pi\) the set of all such history‐dependent policies. We denote by $V_{\mathcal{M}, h}^{\pi} \colon \mathcal{S} \xrightarrow{} \mathbb{R}$ the value function at step $h$ under policy $\pi$ on MDP $\mathcal{M}$, i.e., \; \(V_{\mathcal{M}, h}^{\pi}(s) := \mathbb{E}_{\mathcal{M}, \pi} \left[ \sum_{t=h}^{H} R(s_t, a_t) \,\middle|\, s_h = s \right]\footnote{Since the policy $\pi$ is allowed to be non Markovian, this quantity can be defined using the history $H_h = \ens{s_1, \dots, s_h}$ as follows: $\mathbb{E}_{\mathcal{M}, \pi} \left[ \sum_{t=h}^{H} R(s_t, a_t) \,\middle|\, s_h = s \right] = \E_{H_h \mid s_h = s} \mathbb{E}_{\mathcal{M}, \pi} \left[ \sum_{t=h}^{H} R(s_t, a_t) \,\middle|\, H_h \right].$ }.\) We use $\pi_{\mathcal{M}}^{\star}$ to denote the optimal policy for the MDP $\mathcal{M}$, and $V^{\star}_{\mathcal{M}, h}$ to denote the optimal value under the optimal policy at step $h$.

We fix a \emph{simulator class}
\(\mathcal U=\{\mathcal M_\xi:\xi\in\Xi\subset\mathbb R^{d}\}\)
of candidate MDPs that share $(\mathcal S,\mathcal A,R,H,s_1)$
but can differ in $P_{\mathcal M}$ via the physical parameter
vector~\(\xi\).  The unknown \emph{real-world} environment is
\(\mathcal M^\star=\mathcal M_{\xi^\star}\in\mathcal U\).  We assume full
observability and that the learner can interact freely with any
\(\mathcal M\in\mathcal U\) in simulation, but never observes
\(\xi^\star\) directly.

\paragraph{Sim-to-real Transfer Problem}
Given access to the simulators in $\mathcal U$, the goal is to
output a policy $\pi$ that attains high return when
executed in the real-world MDP $\mathcal M^\star$.
We quantify performance via the
\emph{sim-to-real gap} which is defined as the difference between the value of the learned policy $\pi$ during the simulation phase (or training phase), and the value of an optimal policy for the real world, i.e. 
\[\mathrm{Gap}(\pi) := V_{\mathcal{M}^\star, 1}^{\star}(s_1) -  V_{\mathcal{M}^\star, 1}^{\pi}(s_1).\]

\paragraph{Domain Randomization}

Domain randomization specifies a prior distribution
\(\nu\) over parameters \(\Xi\) and thus over \(\mathcal U\).
Sampling $\xi\sim\nu$ at the start of every episode
induces a \emph{latent MDP} (LMDP) whose optimal Bayes policy is
\[
  \pi_{\mathrm{DR}}^{\star}:=
  \arg\max_{\pi\in\Pi}\;
  \mathbb E_{\xi\sim\nu}\!\bigl[V_{\mathcal M_\xi,1}^{\pi}(s_1)\bigr].
\]
In practice we approximate $\pi_{\mathrm{DR}}^{\star}$ with any RL
algorithm that trains in the simulator while resampling
$\xi\!\sim\!\nu$ each episode.

\paragraph{Offline Domain Randomization}

ODR assumes an offline data set
\(\mathcal D=\{(s_i,a_i,s'_i)\}_{i=1}^{N}\)
of i.i.d.\ transitions collected in the real system
\(\mathcal M^\star\) under some unknown behavior policy.
The aim is to estimate a distribution \(p^\star\) over
\(\Xi\) that explains the data and can later be used for
policy training.
We restrict \(p_\phi(\xi)=\mathcal N(\mu,\Sigma)\)\footnote{The Gaussian parameterization over $\xi$ is only a modeling choice, not a mathematical necessity. Any other parametric family for $P_\phi$ that satisfies our upcoming assumptions could be substituted without changing the arguments.}
and learn $\phi$ by maximum likelihood:

\begin{align}
    p^\star(\xi) &= \argmax_{p_\phi(\xi)} \prod_{(s_t, a_t, s_{t+1}) \in \mathcal{D}}  \E_{\xi \sim p_\phi(\xi)} \left[ P_\xi (s_{t+1} \mid s_t, a_t) \right] \\
    &= \argmax_{p_\phi(\xi)} \sum_{(s_t, a_t, s_{t+1}) \in \mathcal{D}} \log\left[ \E_{\xi \sim p_\phi(\xi)} \left[ P_\xi (s_{t+1} \mid s_t, a_t) \right] \label{eq:formal-problem} \right].
\end{align}
We justify that this formulation is well-motivated in w\Cref{app:insights-odr}.

Finally, we train a policy with the learned distribution:
\[
  \pi_{\mathrm{ODR}}^{\star}:= \argmax_{\pi\in\Pi} \mathbb E_{\xi\sim p^\star} \bigl[V_{\mathcal M_\xi,1}^{\pi}(s_1)\bigr],
\]
expecting $\pi_{\mathrm{ODR}}^{\star}$ to transfer with lower gap
thanks to the data-informed parameter distribution.

A conceptual comparison between Uniform Domain Randomization and Offline Domain Randomization is illustrated in \Cref{fig:udr_vs_odr}.
\begin{figure}
    \centering
    \includegraphics[width=\linewidth]{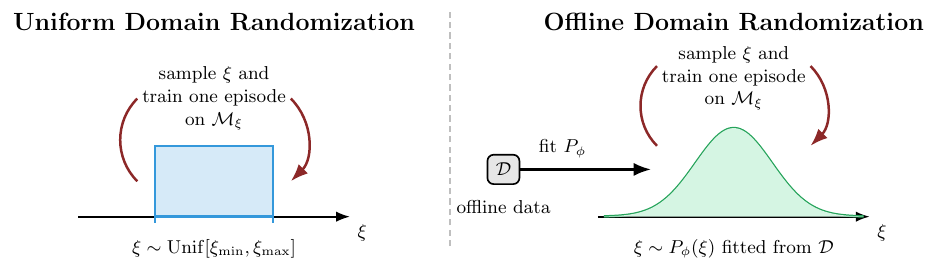}
    \caption{Conceptual comparison between Uniform Domain Randomization (left) and Offline Domain Randomization (right).}
    \label{fig:udr_vs_odr}
\end{figure}

\section{Weak Consistency of the ODR Estimator}\label{sec:main-results}

\subsection{Technical Assumptions}\label{sec:assumptions-weak}
Before stating the theoretical guarantees for ODR, we introduce some mild assumptions of regularity and identifiability that will be useful for our proofs.

The following assumption assures that $P_\xi$ is regular in the following sense.
\begin{assumption}[Simulator Regularity]\label{A.2}
    There exists a $\sigma$-finite measure $\lambda$ on $\mathcal S$
    and a constant $K<\infty$ such that for all $\xi\in\Xi$ and
    $(s,a,s')$
    \begin{equation}
        P_\xi(ds'\mid s,a) = p_\xi(s'\mid s,a)\,\lambda(ds'),\quad 0 \leq p_\xi(s'\mid s,a)\leq K,
    \end{equation}
    and $\xi\mapsto p_\xi(s'\mid s,a)$ is continuous.
\end{assumption}

Notice that when $\mathcal S$ is finite, and $\lambda$ is the counting measure on $\mathcal S$, then the first assumption clearly holds with $K = 1$ because $p_\xi(s' \mid s, a) = P_\xi(\ens{s'} \mid s, a) \leq 1$. In this case, it suffices for the mass probability to depend continuously on $\xi$ in order to verify \Cref{A.2}. Another case where this continuity holds is the Gaussian case $p_\xi(s' \!\! \mid \!\! s, a) = \mathcal{N}(s'; A(\xi) s + B(\xi) a, C(\xi))$, where $A(\xi), B(\xi), C(\xi)$ are matrices that vary continuously in $\xi$.

\begin{assumption}[Parameter-Space Compactness]\label{A.3}
    We fit $\phi=(\mu,\Sigma)$ in $\Phi=\{\mu\in \Tilde \Xi : 0 \preceq \Sigma \preceq \sigma_{\max} I\}$ where $\Tilde \Xi$ is compact, hence $\Phi$ is compact in the product topology.
\end{assumption}

This is a natural assumption, since in practice one always has prior bounds on each physical parameter, yielding a known compact search region.

Furthermore, we assume that all the transitions that appear in our dataset correspond to positive mixture probability. More formally,

\begin{assumption}[Mixture Positivity]\label{A.4}
    There exists some constant $c > 0$ such that the induced kernel
    \begin{align}
        q_\phi(s'\mid s,a) := \E_{\xi \sim P_\phi(\xi)}\left[ p_\xi(s' \mid s, a) \right] = \int p_\xi(s'\mid s,a)\,P_\phi(d\xi),
    \end{align} 
    satisfies $q_\phi(s'\mid s,a) \geq c > 0$ for every  $(s,a,s')\in \D$ and every $\phi\in\Phi$.
\end{assumption}

This guarantees that every transition in the dataset lies within the support of the simulator under the learned domain randomization distribution, so the log-likelihood is always well defined. 


Furthermore, we assume that the only mixture distribution which exactly recovers the true transition kernel is the degenerate distribution concentrated at the true parameters $\xi^\star$.

\begin{assumption}[Identifiability]\label{A.5}
    Let $\mu$ be the dataset's distribution. If for $\mu$-almost every $(s, a)$ $ q_\phi(\cdot\mid s,a) = p_{\xi^\star}( \cdot\mid s,a)$, then $\phi=(\xi^\star, 0)$.
\end{assumption}

\subsection{Notation for ODR}
Throughout this work, we use a capital letter, $P$, to denote a probability distribution, and the corresponding lowercase letter, $p$, to denote its probability density (or mass) function.

We define the empirical and population log-likelihoods by
\begin{align}
    L_N(\phi)&:=\frac{1}{N}\sum_{i=1}^N a\bigl(X_i, \phi\bigr), \quad L(\phi) := \mathbb{E}_{X\sim P_{\xi^\star}}\bigl[a(X, \phi)\bigr],
\end{align}
where $X_i = (s_i, a_i, s'_i)$ is the $i$-th transition in $\D$, and $X = (s, a, s')$ is a generic transition. The function $a$ is defined by 
\begin{equation}
    a(x, \phi) := \log q_\phi(s' \mid s, a) = \log \int_\xi p_\xi(s' \mid s, a) p_\phi(\xi) d\xi.
\end{equation}

\subsection{Main Theorem}

The first lemma proves the uniqueness of the maximizer of the population log-likelihood $L$. A detailed proof of this lemma can be found in \Cref{app:proofs}.

\begin{lem}[Uniqueness of the Population Maximizer]\label{lemma:unique-maximizer}
    Under assumptions \ref{A.2}, \ref{A.4} and \ref{A.5}, the population log-likelihood 
    \[ L(\phi) =\E_{(s,a,s')\sim P_{\xi^\star}}\bigl[\log q_\phi(s'\mid s,a)\bigr] \]
    where ${\displaystyle q_\phi(s'\mid s,a)=\int P_\xi(s'\mid s,a)\,P_\phi(d\xi)}$, has the unique maximizer $\phi^\star=(\mu^\star,\Sigma^\star)=\bigl(\xi^\star,\,0\bigr)$.
\end{lem}

We now state our first consistency result for ODR.

\begin{thm}[Weak Consistency of ODR]\label{thm:consistency}
    Under Assumptions ~\ref{A.2}, \ref{A.3}, \ref{A.4} and \ref{A.5}, any measurable maximizer 
            \(\displaystyle \widehat{\phi}_N\in\argmax_{\phi\in\Phi}L_N(\phi)\) 
            satisfies $\widehat{\phi}_N \;\xrightarrow[N\xrightarrow{} \infty]{P}\;\phi^\star$.
\end{thm}

\Cref{thm:consistency} guarantees that with a sufficiently large offline dataset, ODR recovers a distribution arbitrarily close to the true parameter $\xi^\star$.

The following lemma is particularly strong: it establishes uniform convergence in probability of $L_N$.
\begin{lem}\label{lemma:L continu}
    The function $\phi \mapsto L(\phi)$ is uniformly continuous on $\Phi$, and furthermore
    \begin{equation}
        \sup_{\phi \in \Phi} \left|L_N(\phi) -L(\phi) \right| \xrightarrow[N \to \infty]{P} 0.
    \end{equation}
\end{lem}
The proof of this lemma relies on a \textit{uniform law of large numbers} (ULLN) -in particular the ULLN for \textit{Glivenko-Cantelli} classes from \cite{NEWEY19942111}- and is deferred to \Cref{app:proofs}. In contrast, the ordinary law of large numbers only guarantees that for each \textit{fixed} $\phi$ one has $L_N(\phi) \to L(\phi)$ in probability, i.e., $\left| L_N(\phi) - L(\phi) \right| \to 0$ for that particular $\phi$. This pointwise convergence does \textit{not} imply that $\sup_{\phi \in \Phi} \left|L_N(\phi) -L(\phi) \right| \to 0$ , which is exactly what the ULLN provides. Uniform convergence over all $\phi \in \Phi$ is crucial to control the behavior of the empirical maximizers and hence to establish the consistency of our estimator.

The following lemma formalizes a uniform separation property: any parameter $\phi$ lying outside an $\epsilon$-neighborhood of the true maximizer $\phi^\star$ must have its population log-likelihood at least $\eta > 0$ below $L(\phi^\star)$.
\begin{lem}\label{lemma:infimum-g}
    Let $\phi^\star$ be the unique maximizer of $L$. We have 
    \begin{equation}
        \forall \epsilon > 0, \exists \eta(\epsilon) > 0, \forall \phi \in \Phi, \norm{\phi^\star - \phi} \geq \epsilon \implies L(\phi^\star) - L(\phi) \geq \eta(\epsilon) > 0.
    \end{equation}
\end{lem}
The proof of \Cref{lemma:infimum-g} is deferred to \Cref{app:proofs}.

\begin{proof}[Proof of~\Cref{thm:consistency}]
    We consider a sequence of measurable maximizers $\widehat{\phi}_N \in \argmax_{\phi\in\Phi} L_N(\phi)$.
    Let $\epsilon > 0$ be a fixed positive real number. Our goal is to prove that 
    \begin{equation}
        P\left( \norm{\widehat{\phi}_N - \phi^\star} \geq \epsilon \right) \xrightarrow[N \to \infty]{} 0.
    \end{equation}
    Using Lemma \ref{lemma:infimum-g}, we conclude that there exists some $\eta > 0$ such that
    $\forall \phi \in \Phi$ if $\norm{\phi^\star - \phi} \geq \epsilon$ then $L(\phi^\star) - L(\phi) \geq \eta > 0. \label{ineq:infimum}$
    Now, let $E_{\eta}$ be the event 
    \begin{equation}
    E_\eta = \{\sup_{\phi \in \Phi} \left| L_N(\phi) - L(\phi) \right| < \eta / 3\} 
    \end{equation} then under $E_\eta$, if $\norm{\phi^\star - \phi} \geq \epsilon$ we have 
        \begin{align}
            L_N(\phi^\star) =  L_N(\phi^\star) - L(\phi^\star) + L(\phi^\star)
            \geq - \left| L_N(\phi^\star) - L(\phi^\star) \right|   + L(\phi^\star)  \geq - \eta / 3 +  L(\phi^\star),
        \end{align} 
        since under $E_\eta$, $- \left| L_N(\phi) - L(\phi) \right| \geq -\eta / 3$, similarly 
        \begin{equation}
            L(\phi^\star) \geq  L(\phi) + \eta =  L(\phi) - L_N(\phi) + L_N(\phi) + \eta \geq- \left| L_N(\phi) - L(\phi) \right| + L_N(\phi) + \eta.
        \end{equation} 
    and combining these two inequalities gives
    \(L_N(\phi^\star) \geq L_N(\phi) + \eta / 3.\)
    This proves that, under $E_\eta$, $\hat \phi_N \in \mathrm{B}(\phi^\star, \epsilon) := \enstq{\phi \in \Phi }{ \norm{\phi - \phi^\star} < \epsilon}$ thus
    \( \{\|\widehat{\phi}_N - \phi^\star\| \geq \epsilon \} \subset E_\eta^c,\)
    which yields 
    \begin{equation}
        P(\|\widehat{\phi}_N - \phi^\star\| \geq \epsilon) \leq P\Big ( \sup_{\phi \in \Phi} \left| L_N(\phi) - L(\phi) \right| \geq \eta / 3 \Big) \xrightarrow[n \to \infty]{\text{By \Cref{lemma:L continu}}} 0.
    \end{equation}  
\end{proof}

The result is a \emph{weak} consistency statement ($\widehat{\phi}_N\to\phi^\star$ in probability). In \Cref{sec:strong-consistency} we strengthen this to almost-sure convergence by adding a Lipschitz regularity assumption.

\section{Strong Consistency under Uniform Lipschitz Conditions}\label{sec:strong-consistency}

While \Cref{thm:consistency} guarantees that the ODR estimate converges in probability to the true parameter distribution, in many practical settings one desires a stronger, almost sure guarantee.  Intuitively, \emph{strong consistency} asserts that, with probability one, the estimated distribution will converge exactly to the true one as more offline data is observed.  In this section we show that, under an additional Lipschitz continuity assumption on the log‐likelihood function, ODR enjoys this almost-sure convergence property.

\subsection{Addiotional Assumption}

The key extra ingredient is a uniform control over how rapidly the single step log-likelihood 
\(  a\bigl(x,\phi\bigr) \)
can change as we vary the distributional parameter $\phi = (\mu,\Sigma)$. Formally:
\begin{assumption}[Uniform Lipschitz Continuity]\label{A.6}
  There exists a constant $L<\infty$ such that for every transition $x=(s,a,s')$ and all $\phi,\psi\in\Phi$, we have
  \(
    \bigl|a(x,\phi)-a(x,\psi)\bigr|
    \le
    L\,\bigl\|\phi-\psi\bigr\|_2.
  \)
\end{assumption}


This condition ensures that the family $\enstq{a(\cdot,\phi)}{\phi\in\Phi}$ is \emph{equi-Lipschitz}, which -together with compactness of $\Phi$- yields a \textit{uniform strong law of large numbers}. In turn, this uniform convergence of the empirical log-likelihood to its population counterpart underpins the almost sure convergence of the maximizers.

\subsection{Main Result}
We can now state our strong consistency result:
\begin{thm}[Strong Consistency of ODR]\label{thm:strong-consistency}
  Under \Cref{A.2,A.3,A.4,A.5,A.6}, let \(
    \widehat{\phi}_N
    \in
    \argmax_{\phi\in\Phi}L_N(\phi)
  \)
  be any measurable maximizer of the empirical log-likelihood, then
  \begin{equation}
          \widehat{\phi}_N
    \xrightarrow[N\to\infty]{\mathrm{a.s.}}\;
    \phi^\star
    =
    (\xi^\star,0),
  \end{equation}

  i.e., almost surely the estimated distribution collapses exactly onto the true simulator parameters.
\end{thm}

The distinction between convergence \emph{in probability} and \emph{almost surely} is subtle but meaningful: almost-sure consistency implies that, except on a set of histories of measure zero, as soon as enough data is collected the optimizer will \emph{never} stray from the true maximum again. In contrast, convergence in probability only assures that large deviations become increasingly unlikely.

The heart of the proof is the following uniform strong law, which follows from empirical process arguments once we have the Lipschitz control:

\begin{lem}[Uniform Strong Law]\label{lem:cv-as-D}
  Under \Cref{A.2,A.3,A.4,A.5,A.6}, the empirical and population log-likelihoods satisfy \(
    \sup_{\phi\in\Phi}\bigl|L_N(\phi)-L(\phi)\bigr|
    \;\xrightarrow[N\to\infty]{\mathrm{a.s.}}\;0.
  \)
\end{lem}

Lemma~\ref{lem:cv-as-D} tells us that with probability one the worst-case difference between the finite-sample objective and its ideal limit vanishes. Once this uniform convergence is in hand, classical arguments on continuity and compactness show that the maximizers converge almost surely.

\begin{proof}[Proof (Sketch)]
We first show $\sup_{\phi\in\Phi}|L_N(\phi)-L(\phi)| \xrightarrow[N \to \infty]{\mathrm{a.s.}} 0$ by verifying for each $\epsilon>0$ that \(
\sum_{N} P(\sup_{\phi\in\Phi}|L_N(\phi)-L(\phi)|>2\epsilon)<\infty.\) By compactness of $\Phi$ there is a finite $\epsilon/L$-net $\{\phi_1,\dots,\phi_K\}$ so that Lipschitz continuity gives $|L_N(\phi)-L_N(\phi_i)|+|L(\phi)-L(\phi_i)|\le\epsilon$ whenever $\|\phi-\phi_i\|\le\epsilon/L$. Hence
\begin{equation}
    \Big\{\sup_{\phi}|L_N(\phi)-L(\phi)|>2\epsilon\Big\} \subset \bigcup_{i=1}^K \, \ens{|L_N(\phi_i)-L(\phi_i)|>\epsilon},
\end{equation} and \emph{Hoeffding's inequality} yields
\begin{equation}
    P(|L_N(\phi_i)-L(\phi_i)|>\epsilon)\le2\exp \left(-\frac{N\epsilon^2}{2 \widetilde{M}^2} \right),
\end{equation} where $\widetilde{M} := \max\ens{\abs{\log K}, \abs{\log c}}$. So $P(\sup_{\phi}|L_N(\phi)-L(\phi)|>2\epsilon)\le2K\exp(-cN\epsilon^2)$, which is summable in $N$. \textit{Borel-Cantelli lemma} then gives uniform almost sure convergence. Finally, on the event of uniform convergence one repeats the identification neighborhood argument of~\Cref{thm:consistency} to conclude $\widehat{\phi}_N\to\phi^\star$ almost surely.
\end{proof}

Full details of the proof are deferred to~\Cref{app:proofs-strong}, but the key takeaway is that the Lipschitz assumption upgrades our earlier \emph{in probability} consistency to the far stronger \emph{almost sure} statement, giving robust guarantees for ODR even in worst case data realizations.

\subsection{A Notion of $\alpha$-informativeness}
The strong consistency yields the following.

\begin{lem}\label{lem:informativeness-gaussian}
    Let $\epsilon > 0$. If $\widehat{\phi}_N = (\mu_N, \Sigma_N) \xrightarrow{\rm a.s.}(\xi^\star,0)$ then almost surely there is $N_0$ so that for all $N\ge N_0$, \( P_{\widehat{\phi}_N}\bigl(\mathrm{B}(\xi^\star,\epsilon)\bigr) > \tfrac12. \)
\end{lem}
\begin{proof}[Proof of \Cref{lem:informativeness-gaussian}]
Fix $\epsilon>0$ and let $Z_N \sim\mathcal N(\mu_N,\Sigma_N)$.  Then \(
P\bigl(\|Z_N -\xi^\star\|\ge\epsilon\bigr)
\le
P\bigl(\|Z_N -\mu_N \|\ge\tfrac\epsilon2\bigr)
+
P\bigl(\|\mu_N -\xi^\star\|\ge\tfrac\epsilon2\bigr).
\)
By \textit{Chebyshev's inequality}, \(
P\bigl(\|Z_N -\mu_N \|\ge\tfrac\epsilon2\bigr)
\le
\frac{\E\|Z_N-\mu_N\|^2}{(\epsilon/2)^2}
=\frac{\tr(\Sigma_N)}{(\epsilon/2)^2}.
\)
Hence \(P_{\hat \phi_N}\bigl(\mathrm{B}(\xi^\star,\epsilon)\bigr)
=1-P\bigl(\|Z_N-\xi^\star\|\ge\epsilon\bigr)
\ge 1-\frac{4\tr(\Sigma_N)}{\epsilon^2}-P(\|\mu_N -\xi^\star\|\ge\epsilon/2).\)
As $(\mu_N,\Sigma_N)\to(\xi^\star,0)$ a.s., we have $\|\mu_N-\xi^\star\|\to0$ and $\tr(\Sigma_N)\to0$, so the right hand side tends to $1$ almost surely. Hence $P_{\hat \phi_N}\bigl(\mathrm{B}(\xi^\star,\epsilon)\bigr) \to 1$ almost surely.
\end{proof}

The lemma states that when the estimator $(\mu_N,\Sigma_N)$ converges almost surely to the true mean with vanishing covariance, the Gaussian distribution fitted by ODR eventually assigns {\em more than half of its probability mass} to any fixed $\epsilon$–ball around $\xi^\star$.  In other words, ODR is so informative that the learned randomization concentrates near the real  world. This observation motivates a general, model-agnostic notion of \say{informativeness} for ODR, applicable beyond the Gaussian setting.

\begin{definition}[$\alpha, \epsilon$-Informativeness of an ODR Algorithm $\mathcal{A}$]
\label{def:alpha-informative}
Let $\alpha\in(0,1)$ and $\epsilon > 0$, an algorithm $\mathcal{A}$ is {\normalfont $\alpha, \epsilon$-informative} if there exists almost surely $N_0 \ge 1$ such that for all $N \ge N_0$, running $\mathcal{A}$ on any collection $\mathcal{D} = \{(s_i, a_i, s'_i)\}_{i=1}^N$ of i.i.d. transitions (from the real system) produces an ODR distribution $\widehat{\phi}_{N}$ such that
\[
P_{\widehat\phi_{N}}\bigl(\mathrm{B}(\xi^\star,\epsilon)\bigr) \ge \alpha.
\]
We say algorithm $\mathcal{A}$ is {\normalfont $\alpha$-informative} if $\mathcal{A}$ is $\alpha,\epsilon$-informative for any $\epsilon > 0$.
\end{definition}

 Under this language, Lemma \ref{lem:informativeness-gaussian} states that the Gaussian ODR procedure from \Cref{sec:setting} is \(\alpha\)-informative for
every \(\alpha<1\). When the simulator class \(\Xi\) is finite, \(\alpha\)-informativeness is equivalent to the almost-sure existence of an index \(N_0\) such that, for
all \(N\ge N_0\), the fitted distribution assigns at least \(\alpha\) mass to the singleton \(\{\xi^\star\}\), that is,
\(P_{\widehat\phi_N}(\xi^\star)\ge\alpha\).

\section{Assumptions: Practicality, Violations, and Relaxations}\label{sec:assumptions-relaxations}

\subsection{The i.i.d. Assumption} 
The i.i.d. assumption on the offline dataset $\D$ holds whenever the offline dataset is collected using a \emph{fixed}, \emph{stationary} behavior policy $\pi( \cdot \mid s)$. This assumption is stronger than needed for our weak consistency result: we invoke it only to apply a uniform law of large numbers at the end of the proof of \Cref{lemma:L continu}. As noted after Lemma $2.4$ in \cite{NEWEY19942111}, the same conclusion holds (even for dependent data) for ergodic and strictly stationary sequences $\ens{X_i = (s_i, a_i, s_i')}$ which means that the joint distribution of the vector $(X_i, \dots, X_{i+m})$ does not depend on $i$ for any $m$. This is much weaker than the i.i.d. assumption and is satisfied whenever the offline dataset is collected by a \emph{fixed} behavior policy (not necessarily a stationary policy). In practice, weak consistency should therefore hold broadly.


\subsection{The Mixture Positivity Assumption} 
\Cref{A.4} is a strong requirement: it holds if and only if \(\inf_{x}\inf_{\phi} q_{\phi}(x)>0\), i.e., the density is uniformly bounded away from zero over both \(x\) and \(\phi\). This excludes common light-tailed families (e.g., Gaussian-like), for which \(\inf_{x} q_{\phi}(x)=0\). For \emph{weak consistency}, however, \Cref{A.4} can be relaxed:

\begin{lem}[Relaxation of \Cref{A.4}]
\label{lem:relaxation-ass-3}
Weak consistency of ODR still holds if \Cref{A.4} is replaced by the following tail condition: there exists \( \epsilon_0 > 0 \) such that for all \( \epsilon \in (0,\epsilon_0] \),
\begin{equation}
    P \left(\inf_{\phi} q_{\phi}(X) \le \epsilon \right) \le \frac{1}{\log(1/\epsilon)^{2}}.
\end{equation}
\end{lem}
This assumption is strictly weaker than uniform positivity. The key point is that, to apply the uniform law of large numbers from \cite{NEWEY19942111} in the weak-consistency proof, it suffices to have an \emph{integrable envelope} \( d(x) \) with \( a(x,\phi) \le d(x) \) for all \( \phi \), rather than a uniform bound in \((x,\phi)\), the above tail control yields such an envelope. The proof is deferred to \Cref{sec:relaxation-mixture}.

\subsection{The Uniform Lipschitz Continuity}
Assumption \Cref{A.6} is not immediately interpretable. We give a simple sufficient condition under which it holds:

\begin{lem}[Sufficient Condition for the Uniform Lipschitz Continuity Assumption]\label{lem:CS-lip}
    Suppose the following holds for every $x = (s, a, s')$
    \begin{enumerate}
        \item The function $\xi \mapsto p_\xi(s' \mid s, a)$ is twice continuously differentiable (of class $C^2$),
        \item There exists two constants $G_1 > 0$ and $G_2 > 0$ such that $\left|\nabla_\xi p_\xi(s' \mid s, a)  \right| \leq G_1$ and $|\nabla_\xi^2 p_\xi(s' \mid s, a)  | \leq G_2$ ,
    \end{enumerate}
    then \Cref{A.6} holds with $L = \dfrac{G_1 + G_2 / 2}{c}$.
\end{lem}

A complete proof appears in \Cref{app:cs-lip}. This sufficient condition is easy to interpret because it depends only on the simulator's transition kernel \(p_{\xi}\). In practice, it is satisfied whenever the simulators are governed by smooth physics.

\subsection{The Identifiability Assumption}
\Cref{A.4} is a coverage condition on the dataset: it requires that any mixing Gaussian distribution that reproduces the transition kernel on the state–action pairs observed in \(\mathcal{D}\) must equal the degenerate Dirac mass at the true parameter. Intuitively, the dataset must visit state–action pairs that are informative about \(\xi\). This is information-theoretically minimal: no method can distinguish parameters that are observationally identical on \(\operatorname{supp}(\mu)\).

In the case of partial coverage, we naturally define the \emph{identified set under coverage $\mu$} as follows: 
\begin{equation}
    \mathcal{Q}_\mu^\star := \enstq{\phi \in \Phi}{q_\phi (\cdot \mid s, a) = p_{\xi^\star}(\cdot \mid s, a) \text{ for } \mu-\text{a.e. } (s, a)}.
\end{equation}

It follows from this definition and the proof of \Cref{lemma:unique-maximizer} that: 

\begin{lem}
    The following holds:
    \begin{equation}
        \mathcal{Q}_\mu^\star = \argmax_\phi L(\phi).
    \end{equation}
\end{lem}

Using this notion of identified set, we can generalize \Cref{thm:consistency} when we relax \Cref{A.5} as follows: 
\begin{thm}\label{thm:coverage-relaxed}
    Under Assumptions ~\ref{A.2}, \ref{A.3} and \ref{A.4}, the following holds, Any measurable maximizer 
            \(\displaystyle \widehat{\phi}_N\in\argmax_{\phi\in\Phi}L_N(\phi)\) 
            satisfies $\mathrm{dist} (\widehat{\phi}_N, \mathcal{Q}_\mu^\star) \xrightarrow[N \to \infty]{P} 0$ \footnote{where $\mathrm{dist}$ is the distance to a set defined by $\mathrm{dist}(\phi, \mathcal{Q}) := \inf_{\psi \in \mathcal{Q}} \norm{\phi - \psi}$.}.
\end{thm}

This theorem states that under partial coverage, our estimator does not select a single parameter but converges to an \emph{identified set} of parameters that are observationally indistinguishable on the state–action pairs visited by the data. The proof is deferred to \Cref{apx:partial-coverage}. The proof of this theorem is very general. In particular, even in the misspecified case where $\M^* \notin \U$, we still have $\widehat{\phi}_N \to \phi^\dag \in \argmax_\phi L(\phi)$. A more detailed discussion of the this case is deferred to \Cref{apx:misspecification}.

Without any additional assumptions, the only structural result that we can derive on the identified set is:

\begin{lem}[Upper Hemicontinuity of $\mathcal{Q}_\mu^\star$]\label{lem:Q-nonempty-compact}
    Under Assumptions \ref{A.2}, \ref{A.3} and \ref{A.4} The identified set $\mathcal{Q}_\mu^\star$ is non-empty and compact and and the correspondence $\mu \mapsto \mathcal{Q}^\star_\mu$ is upper hemicontinuous\footnote{A set-valued map $F$ is upper hemicontinuous at $x_0$ if, whenever $x_n\!\to x_0$ and $y_n\in F(x_n)$ with $y_n\to y$, then $y\in F(x_0)$. Equivalently: for every open $U$ with $F(x_0)\subseteq U$, there exists a neighborhood $V$ of $x_0$ such that $F(x)\subseteq U$ for all $x\in V$.} with respect to total variation.
\end{lem}

The proof of this lemma uses \emph{Berge's Maximum Theorem} and is deferred to \Cref{apx:partial-coverage}.

In short, this lemma says if we perturb the dataset’s coverage only slightly (in total-variation distance), the set of maximizers cannot “jump” to a faraway region: any limit of maximizers for the perturbed coverages remains a maximizer at the limit coverage (upper hemicontinuity). Intuitively, modestly adding or reweighting offline data will not create spurious, distant optima, it keeps the solution set nearby, and, as coverage includes more informative state–action pairs, typically makes it tighter.


 The main limitation is that, \emph{without additional assumptions}, we cannot provide a quantitative radius for this set or a Lipschitz-type bound on how much it can move when coverage changes.



\section{Conclusion}
In this paper, we present a rigorous framework for ODR, bridging the gap between empirical success and theoretical understanding in sim-to-real transfer. By casting ODR as maximum likelihood estimation over a parametric family of simulator distributions, we proved that, under mild regularity conditions, the learned distribution is weakly consistent, concentrating on the true dynamics as the offline dataset grows. With the addition of a uniform Lipschitz continuity assumption, we further established strong consistency. Beyond these core results, we scrutinized the practicality of the assumptions and provided diagnostics and relaxations—replacing i.i.d. with stationarity/ergodicity for the ULLN, weakening mixture positivity via a logarithmic tail condition, and giving checkable smoothness criteria that imply the uniform Lipschitz requirement—thereby justifying ODR’s applicability across a broader range of settings. By demonstrating that offline logs are not merely passive datasets but a powerful tool for principled domain randomization, we hope our formulation and analysis can provide insight that paves the way for safer, more data-efficient sim-to-real pipelines in robotics, autonomous vehicles, and beyond.

\newpage

\bibliography{iclr2026_conference}
\bibliographystyle{iclr2026_conference}

\newpage
\appendix

\section{Additional Preliminaries}

\subsection{Refined Analysis of the Uniform DR Sim-to-Real Gap}\label{app:improvement}
In this section, we tighten the worst-case sim-to-real gap bound in the finite, $\delta$-separable setting originally proved by \cite{chen2022understandingdomainrandomizationsimtoreal}.

In the proof of Lemma $5$ of the paper, Inequality $(47)$ yields with probability at least $1 - \delta_0$,

\begin{align}
    \sum_{s^\prime \in \mathcal{H}} \ln\left(\dfrac{P_{\mathcal{M}^\star}(s^\prime \mid s_0, a_0)}{P_{\mathcal{M}_1}(s^\prime \mid s_0, a_0)}\right) \geq \dfrac{n_0 \delta^2}{2} - \log(1 / \alpha) \sqrt{2 n_0 \log \left( 2 / \delta_0 \right)} - \sqrt{n_0 \log\left( 2 / \delta_0 \right) / c} - 2 \alpha S n_0. \label{ineq:1}
\end{align}

The objective is to find a setting of parameters that guarantee with probability at least $1 - \frac{1}{MH}$,
\[\sum_{s^\prime \in \mathcal{H}} \ln\left(\dfrac{P_{\mathcal{M}^\star}(s^\prime \mid s_0, a_0)}{P_{\mathcal{M}_1}(s^\prime \mid s_0, a_0)}\right) > 0.\]

It is sufficient to have the right term positive in \Cref{ineq:1}, i.e.,
\begin{align}
    \dfrac{n_0 \delta^2}{2} - \log(1 / \alpha) \sqrt{2 n_0 \log \left( 2 / \delta_0 \right)} - \sqrt{n_0 \log\left( 2 / \delta_0 \right) / c} - 2 \alpha S n_0 > 0 . \label{ineq:2}
\end{align}

Setting $\alpha = \frac{\delta^2}{8 S}, \delta_0 = \frac{1}{MH}$ (the same values as in the paper), this term becomes
\begin{align}
    \dfrac{n_0 \delta^2}{2} - & \log(1 / \alpha) \sqrt{2 n_0 \log \left( 2 / \delta_0 \right)} - \sqrt{n_0 \log\left( 2 / \delta_0 \right) / c} - 2 \alpha S n_0 \\
    & = \dfrac{n_0 \delta^2}{2} - \log(\frac{8 S}{\delta^2}) \sqrt{2 n_0 \log \left( 2 MH \right)} - \sqrt{n_0 \log\left( 2 MH\right) / c} - \frac{\delta^2}{4} n_0 \\
    & = \frac{n_0 \delta^2}{4} - \log(\frac{8 S}{\delta^2}) \sqrt{2 n_0 \log \left( 2 MH \right)} - \sqrt{n_0 \log\left( 2 MH\right) / c} \\
    & = \sqrt{n_0} \frac{\delta^2}{4} \left[\sqrt{n_0} - \frac{4}{\delta^2} \left(\log(\frac{8 S}{\delta^2}) \sqrt{2 \log \left( 2 MH \right)} - \sqrt{\log\left( 2 MH\right) / c} \right) \right] 
\end{align}

hence the condition \ref{ineq:2} becomes equivalent to 
\begin{equation}
    \sqrt{n_0} > \frac{4}{\delta^2}\sqrt{\log(2 M H)} \left( \sqrt{2}\log\left(\frac{8S}{\delta^2}\right) + \dfrac{1}{\sqrt{c}}\right),
\end{equation} 

or, equivalently, 
\begin{equation}
    n_0 > \frac{16}{\delta^4} \log\left(2MH\right) \left( \sqrt{2}\log\left(\frac{8S}{\delta^2}\right) + \dfrac{1}{\sqrt{c}}\right)^2.
\end{equation}

Thus, there exists a valid setting that satisfies condition \ref{ineq:2} which can be expressed as 
\begin{equation}
    \alpha = \dfrac{\delta ^2}{8S}, \quad  \delta_0 = \dfrac{1}{MH}, \quad n_0 = \dfrac{c_0 \log(MH) \log^{2}(S / \delta^2)}{\delta^4},
\end{equation}

for some constant $c_0 > 0$ sufficiently large.

With this new setting, the result of the Lemma $7$ of the paper becomes 
\begin{equation}
    \mathbb{E}[h_0] \leq O\left(\dfrac{D M^2 \log(MH) \log^{2}(S / \delta^2)}{\delta^4}\right).
\end{equation}

The proof of Theorem $5$ of the paper is not affected by the new expression of $n_0$ and gives 
\begin{equation}
    V_{\mathcal{M^\star}, 1}^\star(s_1) - V_{\mathcal{M^\star}, 1}^{\hat{\pi}}(s_1) \leq O(\mathbb{E}[h_0] + D) = O\left(\dfrac{D M^2 \log(MH) \log^{2}(S / \delta^2)}{\delta^4}\right).
\end{equation}

Combining this result with Lemma 1 of the paper leads to

\begin{equation}
    \mathrm{Gap}(\pi_{DR}^\star,\mathcal{U}) =  O\left(\dfrac{D M^3 \log(MH) \log^{2}(S / \delta^2)}{\delta^4}\right).
\end{equation}

This shows that in the regime where $H$ and $M$ are relatively large, the $O(M^3 \log^3(MH))$ bound of \cite{chen2022understandingdomainrandomizationsimtoreal} can be tightened to $O(M^3 \log(MH))$.

\subsection{Insights into the ODR Objective}\label{app:insights-odr}
In this section, we explain why the formal ODR problem in Equation~\ref{eq:formal-problem} corresponds exactly to fitting the simulator parameter distribution that maximizes the likelihood of our offline dataset.

We seek the parameter $\phi$ of the distribution $P_\phi(\xi)$ that maximizes the probability of observing the triples $(s_i, a_i, s'_i)$ of our dataset, i.e., to solve
\begin{equation}
    \phi^\star = \argmax_{\phi} P\left(\ens{(s_i, a_i, s'_i)}_{i = 1}^N \mid \phi \right),
\end{equation}

This probability corresponds to $P\left(\cap_{i = 1}^N \ens{(s_i, a_i, s'_i)} \mid \phi \right)$ and since the data is i.i.d.,  $\phi^\star$ can be rewritten as follows 
\begin{equation}
    \phi^\star = \argmax_{\phi} \prod_{i = 1}^N P(\ens{(s_i, a_i, s'_i)} \mid \phi).
\end{equation}

Now, we can approximate $P(\ens{(s_i, a_i, s'_i)} \mid \phi)$ by the expected transition probability over all $\xi \sim P_\phi(\xi)$, i.e.,
\begin{equation}
    \phi^\star \approx \argmax_{\phi} \prod_{i = 1}^N \E_{\xi \sim P_\phi(\xi)} \left[ P_\xi(s'_i \mid s_i, a_i)\right].
\end{equation} 

Since the logarithm is increasing, this is equivalent to
\begin{equation}
    \phi^\star \approx \argmax_{\phi} \sum_{i = 1}^N \log \E_{\xi \sim P_\phi(\xi)} \left[ P_\xi(s'_i \mid s_i, a_i)\right],
\end{equation}

which recovers exactly the empirical log‐likelihood objective stated in
Equation~\ref{eq:formal-problem}.

\section{Omitted Proofs in \Cref{sec:main-results}}\label{app:proofs}

\begin{proof}[Proof of \Cref{lemma:unique-maximizer}]
    We have 
    \begin{equation}
        L(\phi) = \E_{(s,a)} \E_{s' \sim p_{\xi^\star}(. \mid s, a)} \Big[ \log \E_{\xi \sim p_\phi(\xi)} [ p_\xi(s' \mid s, a) ] \Big].
    \end{equation}

    We rewrite the inner expectation as follows
    \begin{align}
        \E&_{s' \sim p_{\xi^\star}(. \mid s, a)}  \Big[ \log q_\phi(s' \mid s, a) \Big]= \E_{s' \sim p_{\xi^\star}(. \mid s, a)} \left[ \log \dfrac{q_\phi(s' \mid s, a)}{p_{\xi^\star}(s' \mid s, a)} + \log p_{\xi^\star}(s' \mid s, a) \right].
    \end{align}
    
    Notice that 
    \begin{align}
        \int_{s' \in \mathcal S} q_\phi(s' \mid s, a) \lambda(ds') &= \int_{s' \in \mathcal S} \int_{\xi \in \Xi} p_\xi(s' \mid s, a) p_\phi(d\xi) \lambda(ds'),
    \end{align}
    and using \textit{Fubini-Tonelli's theorem}, it follows, 
    \begin{equation}
        \int_{s' \in \mathcal S} \int_{\xi \in \Xi} p_\xi(s' \mid s, a) p_\phi(d\xi) \lambda(ds')= \int_{\xi \in \Xi} p_\phi(d\xi) \int_{s' \in \mathcal S} p_\xi(s' \mid s, a) \lambda(ds').
    \end{equation}
     Since $p_\xi(\cdot \mid s,a)$ and $p_\phi$ are probability densities, their total mass is $1$, which yields
    \begin{align}
         \int_{s' \in \mathcal S} q_\phi(s' \mid s, a) \lambda(ds') = \int_{\xi \in \Xi} p_\phi(d\xi) = 1. \label{eq:densities}
    \end{align}
    
    Hence $q_\phi(. \mid s, a)$ is a probability density, and one can rewrite $L(\phi)$ using \textit{Kullback-Leibler (KL) divergence} (defined in \cite{10.1214/aoms/1177729694}) as follows

    \begin{align}
        L(\phi) &= \E_{(s,a)} \left[ - D_{KL} \left(p_{\xi^\star}(. \mid s, a) \| q_\phi(. \mid s, a)\right) +\E_{s' \sim p_{\xi^\star}(. \mid s, a)} \left[ \log p_{\xi^\star}(s' \mid s, a) \right] \right] \\
        &= \E_{(s,a)} \left[ - D_{KL} \left(p_{\xi^\star}(. \mid s, a) \| q_\phi(. \mid s, a)\right) \right] + H(\xi^\star),
    \end{align}

    where $H(\xi^\star) = \E_{(s,a)}\E_{s' \sim p_{\xi^\star}(. \mid s, a)} \left[ \log p_{\xi^\star}(s' \mid s, a) \right]$ is independent of $\phi$, and for a fixed $(s, a)$, $D_{KL} \left(p_{\xi^\star}(. \mid s, a) \| q_\phi(. \mid s, a)\right) \geq 0$ with equality if and only if $p_{\xi^\star}(. \mid s, a) = q_\phi(. \mid s, a)$.

    Hence, for all $\phi \in \Phi,  L(\phi) \leq H(\xi^\star)$, and 

    \begin{align}
        L(\phi) = H(\xi^\star) &\iff \E_{(s,a)} \left[ - D_{KL} \left(p_{\xi^\star}(. \mid s, a) \| q_\phi(. \mid s, a)\right) \right] = 0\\
        &\iff \text{For almost every } (s,a), D_{KL} \left(p_{\xi^\star}(. \mid s, a) \| q_\phi(. \mid s, a)\right) = 0 \\
        &\iff \text{For almost every } (s,a), p_{\xi^\star}(. \mid s, a) = q_\phi(. \mid s, a) \\
        &\iff \phi = (\xi^\star, 0), 
    \end{align}

    where the last equivalence follows from \Cref{A.5}. This concludes the proof.
\end{proof}

\begin{proof}[Proof of \Cref{lemma:L continu}]

We begin by stating and proving a few intermediate lemmas that will simplify the proof.

The following lemma states that convergence of $\phi$ implies convergence in distribution of $P_\phi$.
\begin{lem}\label{lemma:cvd}
    Let $\ens{\phi_n} := \ens{(\mu_n, \Sigma_n)} \in \Phi^\N$ a sequence that converges to $\phi := (\mu, \Sigma)$ (i.e. $\norm{\mu_n - \mu} \xrightarrow{} 0$ and $\norm{\Sigma_n - \Sigma}_{\text{op}}\xrightarrow{} 0$). Then $P_{\phi_n}$ converges weakly to $P_\phi$ ($P_{\phi_n} \implies P_\phi$).
\end{lem}

\begin{proof}[Proof of \Cref{lemma:cvd}]
    We denote 
    \begin{equation}
        G_n = \mathcal N (\mu_n, \Sigma_n), \quad G = \mathcal N (\mu, \Sigma).
    \end{equation}

    The characteristic function of $G_n$ is 
    \begin{equation}
        \varphi_{G_n}(t)= \exp\left(i t^\top \mu_n - \frac{1}{2} t^\top \Sigma_n t\right), \qquad t \in \R^d.
    \end{equation}

    For every fixed $t \in \R^d$, we have
    \begin{equation}
        \varphi_{G_n}(t) \xrightarrow[n \xrightarrow{} \infty]{} \exp\left(i t^\top \mu - \frac{1}{2} t^\top \Sigma t\right) = \varphi_G(t).
    \end{equation}

    By \textit{Lévy's continuity theorem} (see \cite{Williams_1991}), we have $P_{\phi_n} \implies P_\phi$.
\end{proof}

Notice that the result holds also in the case where $\Sigma = 0$. In that case, $\varphi_G(t) = exp\left(i t^\top \mu\right)$ which is the characteristic function of the degenerate distribution $\delta_\mu = \mathcal N (\mu, 0)$.

This result will be used to derive the continuity of the function $\phi \mapsto a(x, \phi)$ in the following lemma.

\begin{lem}\label{lemma:continu}
    For some fixed $x = (s,a, s')$ and $\phi \in \Phi$, the function

    \[\phi \mapsto a(x, \phi) := \log \int_{\xi} p_\xi(s' \mid s, a) p_\phi(\xi) d\xi\]

    is continuous on $\Phi$.
\end{lem}

\begin{proof}[Proof of \Cref{lemma:continu}]
    For $\xi \in \Xi$, we denote $h_x(\xi) := p_\xi(s' \mid s, a)$.

    $h_x$ is continuous on $\Xi$ (by Assumption \ref{A.2}) and bounded on $\Xi$, because 
        \begin{align*}
            \forall \xi \in \Xi, \; \left| h_x(\xi) \right| &= \left| p_\xi(s' \mid s, a) \right| \leq M \qquad (\text{again by Assumption } \ref{A.2}).
        \end{align*}

     Let $\ens{\phi_n} := \ens{(\mu_n, \Sigma_n)} \in \Phi^\N$ a sequence that converges to $\phi := (\mu, \Sigma)$. Notice that 
    \begin{equation}
        \int_{\xi} p_\xi(s' \mid s, a) p_{\phi_n}(\xi) d\xi = \E_{P_{\phi_n}}\left[ h_x \right],
    \end{equation}

    and since $P_{\phi_n} \implies P_\phi$ (from Lemma \ref{lemma:cvd}), then $\E_{P_{\phi_n}}\left[ h_x \right] \xrightarrow[n \xrightarrow{} \infty]{} \E_{P_{\phi}}\left[ h_x \right]$.

    We then compose by the logarithm function which is continuous on $(0, \infty)$. This yields $\log \E_{P_{\phi_n}}\left[ h_x \right] \xrightarrow[n \xrightarrow{} \infty]{} \log \E_{P_{\phi}}\left[ h_x \right]$. Equivalently,
    \begin{equation}
        a(x, \phi_n) \xrightarrow[n \xrightarrow{} \infty]{} a(x, \phi).
    \end{equation}

    This concludes the proof by the \textit{sequential characterization of continuity}.
\end{proof}

Now we prove \Cref{lemma:L continu}:

    We have $L_N(\phi) = \dfrac{1}{N} \sum_{i=1}^{N} a(X_i, \phi)$, where $X_i = (s_i, a_i, s'_i) \overset{\mathrm{iid}}{\sim} P_{\xi^\star}$.

    $\Phi$ is compact (by Assumption \ref{A.3}), and by Lemma \ref{lemma:continu}, for each $x$, $\phi \mapsto a(x, \phi)$ is continuous on $\Phi$ .
    
    Additionally, the following holds for any $\phi \in \Phi,$
        \begin{align}
             \left|a(x, \phi)\right| &= \left| \log \int_{\xi} p_\xi(s' \mid s, a) p_\phi(\xi) d\xi \right|
        \end{align}
        
        By Assumptions \ref{A.2} and \ref{A.5}, we have $c \leq \int_{\xi} p_\xi(s' \mid s, a) p_\phi(\xi) d\xi \leq K$. Hence 
        \begin{align}
            \left| a(x, \phi) \right| \leq \Tilde M := \max\ens{\left| \log c \right|, \left| \log K \right|}. \label{eq:bound-a} 
        \end{align}
    
    Since $L(\phi) = \E_{X\sim P_{\xi^\star}}[a(X, \phi)]$, this implies (by Lemma 2.4 from \cite{NEWEY19942111} which is implied by Lemma 1 from \cite{TAUCHEN1985415}) that $L$ is continuous on $\Phi$ and thus uniformly continuous since $\Phi$ is compact by \textit{Heine-Cantor theorem}. Furthermore, 
        \begin{equation}
            \sup_{\phi \in \Phi} \left| L_N(\phi) - L(\phi) \right| \xrightarrow[N \to \infty]{P} 0.
        \end{equation}

\end{proof}

\begin{proof}[Proof of \Cref{lemma:infimum-g}]
    Let $\epsilon > 0$. We consider the set defined as follows
    \begin{equation}
        C_{\phi^\star, \epsilon} := \ens{\phi \in \Phi \mid \norm{\phi - \phi^\star} \geq \epsilon}.
    \end{equation}

    $C_{\phi^\star, \epsilon}$ is compact because it can be written as the intersection of a compact set  
    \begin{equation}
        C_{\phi^\star, \epsilon} = \Phi \; \cap \; f_{\phi^\star}^{-1}\Bigl([\epsilon, \infty )\Bigr),
    \end{equation}

    where we denote $f_{\phi^\star}: \phi \mapsto \norm{\phi - \phi^\star}$. Indeed, $\Phi$ is compact (by Assumption \ref{A.3}) and $f_{\phi^\star}^{-1}([\epsilon, \infty))$ is closed as the inverse image of the closed set $[\epsilon, \infty)$ by the continuous function $f_{\phi^\star}$.

    The function $g: \phi \mapsto L(\phi^\star) - L(\phi)$ is continuous (by Lemma \ref{lemma:L continu}) on the compact set $C_{\phi^\star, \epsilon}$, hence by the \textit{extreme value theorem}, $g$ attains its minimum on $C_{\phi^\star, \epsilon}$ in some $\Tilde \phi \in \Phi$.

    Thus
    \begin{equation}
        \forall \phi \in C_{\phi^\star, \epsilon}, \; L(\phi^\star) - L(\phi) \geq g(\Tilde \phi).
    \end{equation}

    By Lemma \ref{lemma:unique-maximizer}, $g \geq 0$ on $\Phi$ and 
    \begin{equation}
        g(\phi) = 0 \iff \phi = \phi^\star.
    \end{equation}

    Since $\Tilde \phi \neq \phi^\star$ (because $\Tilde \phi \in C_{\phi^\star, \epsilon}$), we have $g(\Tilde \phi) > 0$. Thus, the lemma holds with the choice of $\eta(\epsilon) = g(\Tilde \phi) > 0$.
\end{proof}

\section{Omitted Proofs in \Cref{sec:strong-consistency}}\label{app:proofs-strong}

Before proving \Cref{lem:cv-as-D}, we state and prove a few preliminary lemmas.

\paragraph{Notation for Strong Consistency}
We define the \emph{diameter of $\Phi$} by
\begin{equation}
    \Diam(\Phi) := \sup_{\phi,\psi\in\Phi} \|\phi - \psi\|.
\end{equation}

We begin with the following technical lemma, which gives an upper bound on the number of closed balls of radius $r = \epsilon / L$ needed to cover $\Phi$.

\begin{lem}\label{lem:covering}
Let $0 < \epsilon < 2\,\Diam(\Phi)\,L$, and let $N_\epsilon$ be the minimum number of closed balls of radius $r = \frac{\epsilon}{L}$ required to cover $\Phi$.  Then
\begin{equation}
    N_\epsilon \;\le\; 4^d \Bigl(\frac{\Diam(\Phi)\,L}{\epsilon}\Bigr)^d.
\end{equation}
\end{lem}

\begin{proof}[Proof of \Cref{lem:covering}]
We construct a sequence $\phi_1, \phi_2, \dots$ in $\Phi$ satisfying
\begin{equation}\label{eq:construction}
\forall i \neq j,\quad \|\phi_i - \phi_j\| > r.
\end{equation}
This process must terminate after finitely many steps; denote the final index by $K$.  Indeed, if it were infinite, then compactness of $\Phi$ would yield a convergent subsequence of $\{\phi_n\}$, contradicting \eqref{eq:construction}.

By construction,
\begin{equation}
    \Phi \;\subset\; \bigcup_{k=1}^K \mathrm B(\phi_k,r),
\end{equation}

for otherwise we could pick some $\phi\notin\bigcup_{k=1}^K \mathrm B(\phi_k,r)$ to continue the process, contradicting the definition of $K$.  Hence $N_\epsilon\le K$.

Next, observe that the closed balls $\mathrm B(\phi_k,r/2)$, $k=1,\dots,K$, are pairwise disjoint: if there were
\(\phi\in \mathrm B(\phi_i,r/2)\cap \mathrm B(\phi_j,r/2)\) with \(i\neq j\),
then
\begin{equation}
    \|\phi_i - \phi_j\|
\le \|\phi_i - \phi\| + \|\phi - \phi_j\|
\le r/2 + r/2 = r,
\end{equation}

contradicting \eqref{eq:construction}.

Moreover, for each \(k\),
\begin{equation}
    \mathrm B(\phi_k,r/2)\;\subset\;\mathrm B\bigl(\phi_1,\Diam(\Phi) + r/2\bigr),
\end{equation}

since if \(\|\phi - \phi_k\|\le r/2\) then
\(\|\phi - \phi_1\|\le \|\phi - \phi_k\| + \|\phi_k - \phi_1\| \le r/2 + \Diam(\Phi)\).

Thus 
\begin{equation}
    \bigcup_{k=1}^K \mathrm B(\phi_k,r/2) \;\subset\; \mathrm B\bigl(\phi_1,\Diam(\Phi) + r/2\bigr),
\end{equation}

and by comparing volumes of disjoint balls in $\R^d$ we get
\begin{equation}
    K\;\frac{(r/2)^d\pi^{d/2}}{\Gamma(\tfrac d2+1)} \;\le\; \frac{(\Diam(\Phi) + r/2)^d\pi^{d/2}}{\Gamma(\tfrac d2+1)}.
\end{equation}

Hence
\begin{equation}
    N_\epsilon \;\le\; K \;\le\; \Bigl(1 + \tfrac{2\,\Diam(\Phi)}{r}\Bigr)^d \;\le\; \Bigl(\tfrac{4\,\Diam(\Phi)}{r}\Bigr)^d,
\end{equation}
where the final inequality uses \(\epsilon < 2\,\Diam(\Phi)\,L\).
\end{proof}

In the following two lemmas establish a sufficient condition for the almost sure convergence.

\begin{lem}\label{lem:prob-union-zero}
    Let $(A_\ell)_{\ell \geq 1}$ be a sequence of events. We have 
    \begin{equation}
        P\left(\bigcup_{\ell \geq 1} A_\ell \right) = 0 \iff \forall \ell \geq 1, \quad P(A_\ell) = 0.
    \end{equation}
\end{lem}

\begin{proof}[Proof of \Cref{lem:prob-union-zero}]
    If $P\left(\bigcup_{\ell \geq 1} A_\ell \right) = 1$, then for all $\ell \geq 1$, we have clearly $P(A_\ell) \leq P\left(\bigcup_{\ell \geq 1} A_\ell \right) = 0$ and so $P(A_\ell) = 0$.

    If $P(A_\ell) = 0$ for every $\ell \geq 0$, then we have by \textit{Boole's inequality},
    \begin{equation}
        P\left(\bigcup_{\ell \geq 1} A_\ell \right) \leq \sum_{\ell \geq 0} P(A_\ell) = 0.
    \end{equation}
\end{proof}

\begin{lem}\label{lem:CS-cvas}
    Let $\ens{Z_n}_{n}$ be a sequence of random variables. We have 
    \begin{equation}
        \forall \epsilon > 0, \quad \sum_{n \geq 1} P\left( \left|Z_n \right| \geq \epsilon \right) < \infty \implies Z_n \xrightarrow[n \to \infty]{\rm a.s.} 0.
    \end{equation}
\end{lem}

\begin{proof}[Proof of \Cref{lem:CS-cvas}]
    We have by definition of the almost sure convergence, $Z_n \xrightarrow[n \to\infty]{\rm a.s.} 0$ if and only if $\displaystyle{P\left( \lim_{n \xrightarrow{} \infty} Z_n = 0\right) = 1}$. Equivalently, 
    \begin{equation}
        P\left(\forall \epsilon > 0, \exists n \geq 1, \forall m \geq n, \left| Z_n \right| < \epsilon \right) = 1,
    \end{equation}

    and since we can replace $\epsilon$ by any sequence of positive real numbers that converges to $0$, the previous condition is equivalent to 
    \begin{equation}
        P\left( \bigcap_{\ell \geq 1} \bigcup_{n \geq 1} \bigcap_{m \geq n} \ens{\left| Z_n \right| < \frac{1}{\ell}} \right) = 1.
    \end{equation}

    Considering the complementary event, this is equivalent to 
    \begin{equation}
        P\left( \bigcup_{\ell \geq 1} \bigcap_{n \geq 1} \bigcup_{m \geq n} \ens{\left| Z_n \right| \geq \frac{1}{\ell}} \right) = 0.
    \end{equation}

    Using \Cref{lem:prob-union-zero}, in order to have the almost sure convergence of $Z_n$ to $0$, it is sufficient to prove that
    \begin{equation}
        \forall \ell \geq 1, \quad  P\left( \bigcap_{n \geq 1} \bigcup_{m \geq n} \ens{\left| Z_n \right| \geq \frac{1}{\ell}} \right) = 0.
    \end{equation}

    Now suppose that for all $\epsilon > 0$, $\sum_{n \geq 1} P\left( \left|Z_n \right| \geq \epsilon \right) < \infty$. This implies that for all $\ell \geq 1$, we have 
    \begin{equation}
        \sum_{n \geq 1} P\left( \left|Z_n \right| \geq \dfrac{1}{\ell} \right) < \infty.
    \end{equation}

    Using \textit{Borel-Cantelli lemma}, this implies that 
    \begin{equation}
        \forall \ell \geq 1, \quad P\left(\bigcap_{n \geq 1} \bigcup_{m \geq n} \ens{\left|Z_n \right| \geq \dfrac{1}{\ell}} \right) = 0.
    \end{equation}

    This concludes the proof.
\end{proof}

\begin{lem}\label{lem:hoeffding}
    For any fixed $\phi \in \Phi$, and $\epsilon > 0$, we have 
    \begin{equation}
       P\left(\left| L_N(\phi) - L(\phi) \right| \geq \epsilon \right) \leq  2 \exp\left( -\dfrac{N \epsilon^2}{2 \Tilde M^2} \right). 
    \end{equation} 
\end{lem}

\begin{proof}[Proof of \Cref{lem:hoeffding}]
    We have 
    \begin{equation}
        L_N(\phi) = \frac{1}{N} \sum_{i = 1}^N a(X_i, \phi), \quad L(\phi) = \E_{X\sim P^\star} \left[a(X, \phi)\right],
    \end{equation}

    where $X, X_1, \dots, X_N \overset{\mathrm{iid}}{\sim} P_{\xi^\star}.$

    We already establish that $\left| a(x, \phi) \right| \leq \Tilde M$ (see \ref{eq:bound-a}), hence
    \begin{align}
         P\left(\left| L_N(\phi) - L(\phi) \right| \geq \epsilon \right) &=  P\left(\left| \sum_{i = 1}^N \left( a(X_i, \phi) - \E_{X\sim P^\star} \left[a(X, \phi)\right] \right) \right| \geq N \epsilon \right)\\
         & \leq 2 \exp{\left(-\dfrac{2 (N \epsilon)^2}{\sum_{i=1}^N (2 \Tilde M)^2} \right)} \label{ineq:hoeffd} \\
         &\leq  2 \exp\left( -\dfrac{N \epsilon^2}{2 \Tilde M ^2} \right),\label{ineq:hoeffding-fixed-phi}
    \end{align}
    where Inequality \ref{ineq:hoeffd} results from \textit{Hoeffding's inequality}.
\end{proof}

\begin{proof}[Proof of \Cref{lem:cv-as-D}]
    Let $0 < \epsilon < 2DL$. We cover $\Phi$ by $N_\epsilon$ closed balls of radius $r = \epsilon / L$, i.e., \[\Phi \subset \bigcup_{k = 1}^{N_\epsilon} \mathrm B (\phi_k, r),\] for some $\phi_1, \dots, \phi_{N_\epsilon} \in \Phi$, where $N_\epsilon \leq 4^d \left( \dfrac{DL}{\epsilon} \right)^d$ by \Cref{lem:covering}. 

    For all $\phi \in \Phi$, there exists an integer $1 \leq k(\phi) \leq N_\epsilon$ such that $\norm{\phi - \phi_{k(\phi)}} \leq r$, hence it follows from \Cref{A.6} that 
    \begin{align}
        \forall x, \quad \norm{a(x, \phi) - a(x, \phi_{k(\phi)})} \leq L \norm{\phi - \phi_{k(\phi)}} \leq L r = \epsilon . \label{eq:lips}
    \end{align}

    We have 
    \begin{align*}
        \left| L_N(\phi) - L(\phi) \right| &\leq \left| L_N(\phi) - L_N(\phi_{k(\phi)}) \right| + \left| L_N(\phi_{k(\phi)}) - L(\phi_{k(\phi)}) \right| + \left| L(\phi_{k(\phi)}) - L(\phi) \right|.
    \end{align*}

    The first term can be bounded using Inequality \ref{eq:lips} as follows,
    \begin{equation}
        \left| L_N(\phi) - L_N(\phi_{k(\phi)}) \right| \leq \dfrac{1}{N} \sum_{i = 1}^N \left|a(X_i, \phi) - a(X_i, \phi_{k(\phi)}) \right| \leq \epsilon.
    \end{equation}

    Similarly, the third term satisfies 
    \[\left| L(\phi_{k(\phi)}) - L(\phi) \right| = \left| \E_X a(X, \phi_{k(\phi)}) - \E_X a(X, \phi) \right| \leq \E_X \left| a(X, \phi_{k(\phi)}) - a(X, \phi)\right| \leq \epsilon,\]

    where the first equality holds from \textit{Jensen's inequality}.

    Putting these inequalities together yields 
    \begin{equation}
        \sup_{\phi \in \Phi}\left| L_N(\phi) - L(\phi) \right| \leq \max_{i = 1, \dots, N_\epsilon} \left| L_N(\phi_i) - L(\phi_i) \right| + 2 \epsilon.
    \end{equation}

    This implies that 
    \begin{align}
        P\left(\sup_{\phi \in \Phi}\left| L_N(\phi) - L(\phi) \right| \geq 3 \epsilon \right) &\leq P\left( \max_{i = 1, \dots, N_\epsilon} \left| L_N(\phi_i) - L(\phi_i) \right| \geq \epsilon \right) \\
        & \leq \sum_{i = 1}^{N_\epsilon} P\left( \left| L_N(\phi_i) - L(\phi_i) \right| \geq \epsilon \right) \label{ineq:union-bound}\\
         &  \leq \sum_{i = 1}^{N_\epsilon} 2 \exp\left( -\dfrac{N \epsilon^2}{2 M^2} \right) \label{eq:application-hoeffding} \\
        &= 2 N_\epsilon \exp\left( - \dfrac{N \epsilon^2}{2 M^2} \right)  \\
        &\leq 2 \cdot 4^d \left( \dfrac{DL}{\epsilon} \right)^d \exp\left( - \dfrac{N \epsilon^2}{2 M^2} \right)
        \end{align}

    where \Cref{ineq:union-bound} uses union bound, \Cref{eq:application-hoeffding} follows from \Cref{lem:hoeffding} and the last inequality follows from \Cref{lem:covering}.

    This yields when $N \to \infty$ 
    \begin{equation}
        P\left(\sup_{\phi \in \Phi}\left| L_N(\phi) - L(\phi) \right| \geq 3 \epsilon \right) = o\left( \dfrac{1}{N^2} \right). 
    \end{equation}

    This assures that $\sum_{N \geq 1} P\left( \sup_{\phi \in \Phi} \left| L_N(\phi) - L(\phi) \right| \geq 3 \epsilon \right) < \infty$, which gives by \Cref{lem:CS-cvas}:
    \begin{equation}
        \sup_{\phi \in \Phi} \left| L_N(\phi) - L(\phi) \right| \xrightarrow[N \to \infty]{\rm a.s.} 0.
    \end{equation}
\end{proof}

\begin{proof}[Proof of \Cref{thm:strong-consistency}]

By the preceding lemma we have the event
\begin{equation}
     P\left( \Omega_0 :=\Bigl\{\,
        \omega:
        \sup_{\phi\in\Phi}|L_N(\phi,\omega)-L(\phi)|
        \xrightarrow[N\to\infty]{}0
     \Bigr\} \right) = 1. 
\end{equation}

Fix $\omega\in\Omega_0$ and, let $\epsilon > 0$. From \Cref{lemma:infimum-g}, there exists $\eta > 0$ such that
    \begin{equation}\label{eq:gap}
        \forall \phi \in \Phi, \quad  \norm{\phi^\star - \phi} \geq \epsilon \implies L(\phi^\star) - L(\phi) \geq \eta > 0.
    \end{equation}

Since $\omega\in\Omega_0$, there exists a random index $N_0(\omega,\eta)$ with
\begin{equation}\label{eq:unifdev}
  \sup_{\phi\in\Phi}|L_N(\phi)-L(\phi)| < \eta / 3
  \quad\forall N\ge N_0(\omega,\eta).
\end{equation}

Take $N\ge N_0(\omega,\eta)$ and suppose, towards a contradiction, that
$\|\widehat{\phi}_N(\omega) - \phi^\star\|\ge \epsilon$.  Then, using
\eqref{eq:gap}--\eqref{eq:unifdev},
\begin{equation}
    L_N(\widehat{\phi}_N(\omega))
     \le L(\widehat{\phi}_N(\omega))+\eta /3
     \le L(\phi^\star)- \eta+\eta / 3
     =   L(\phi^\star)- 2 \eta / 3 \leq L_N(\phi^\star) - \eta / 3
     < L_N(\phi^\star),  
\end{equation}

which contradicts the maximality of $\widehat{\phi}_N(\omega)$.
Hence, for all $N\ge N_0(\omega,\eta)$,
\begin{equation}\label{eq:delta-band}
  \|\widehat{\phi}_N(\omega) -\phi^\star\|<\epsilon.
\end{equation}
This implies that $\Omega_0 \subset \ens{\omega: \widehat{\phi}_N (\omega) \xrightarrow[N \to \infty]{} \phi^\star}$. Since $P(\Omega_0) = 1$, we conclude

\begin{equation}
    \widehat{\phi}_N \xrightarrow[N\to\infty]{\rm a.s.} \phi^\star.
\end{equation}

 
\end{proof}

\section{Omitted Proofs in \Cref{sec:assumptions-relaxations}}\label{sec:proofs-relaxations}

\subsection{Relaxation of the Mixture Positivity Assumption}\label{sec:relaxation-mixture}
\begin{restate-lem}{\ref{lem:relaxation-ass-3}}
    The weak consistency of ODR still holds if we replace \Cref{A.4} with the following (weaker) assumption:
    \begin{equation}
        P\left( \inf_{\phi} q_\phi(x)  \leq \epsilon \right) \leq \frac{1}{(\log(\epsilon))^2 } \quad \text{ for } \epsilon \text{ sufficiently small}.
    \end{equation} 
\end{restate-lem}

\begin{proof}[Proof of \Cref{lem:relaxation-ass-3}]
    We start by proving these two elementary lemmas.
    \begin{lem}\label{lem:esperance-positive}
        For any almost surely non-negative random variable $Z$, i.e., $P(Z \geq 0) = 1$, we have
        \begin{equation}
            \E[Z] = \int_{0}^\infty P(Z \geq \alpha) \di \alpha.
        \end{equation}
    \end{lem}
    
    \begin{proof}[Proof of \Cref{lem:esperance-positive}]
        We have 
        \begin{align}
            \int_{0}^\infty P(Z \geq \alpha) \di \alpha &= \int_{0}^\infty \E[\mathbbm{1}_{Z \geq \alpha}] \di \alpha \\
            &= \int_{\alpha = 0}^\infty \int_{z = 0}^\infty \mathbbm{1}_{z \geq \alpha} \di P(z) \di \alpha \\
            &= \int_{z = 0}^\infty \left[ \int_{\alpha = 0}^\infty \mathbbm{1}_{z \geq \alpha} \di \alpha  \right]\di P(z)  \label{eq:tonelli-intervert} \\
            &= \int_{z = 0}^\infty \left[ \int_{\alpha = 0}^z 1 \di \alpha  \right]\di P(z)  \\
            &= \int_{z = 0}^\infty z \di P(z) \\
            &= \E[Z], \label{eq:positive-ra}
        \end{align}
        where Equality \eqref{eq:tonelli-intervert} follows from \emph{Fubini-Tonelli's theorem}, and Equality \eqref{eq:positive-ra} follows from the non-negativity of the random variable $Z$.
    \end{proof}
    
    \begin{lem}\label{lem:sup-log}
        For any positive function $f:I \to (0, \infty)$ defined on some interval $I\subset \R$, we have 
        \begin{equation}
            \sup_{x} \log f(x) = \log \sup_{x} f(x).
        \end{equation}
    \end{lem}
    
    \begin{proof}[Proof of \Cref{lem:sup-log}]
        For any $x \in I$ we have by monotonicity of the logarithm function
        \begin{align}
            \log f(x) \leq \log \sup_{x} f(x),
        \end{align}
        hence, $\sup_{x} \log f(x) \leq \log \sup_{x} f(x)$. Furthermore,
        \begin{align}
            f(x) = e^{\log f(x)} \leq e^{\sup_{x} \log f(x)},
        \end{align}
        and taking the supremum over $x \in I$ yields $\sup_{x} f(x) \leq e^{\sup_{x} \log f(x)}$, thus 
        \begin{equation}
            \log \sup_{x} f(x) \leq sup_{x} \log f(x),
        \end{equation} 
        which concludes the proof.
    \end{proof}
    
    Note that the only passage of the proof of \Cref{thm:consistency} in which we use \Cref{A.4} is when we derive a uniform bound on the function $a$ in Inequality \ref{eq:bound-a}. More precisely, we proved that 
    \begin{equation}
        \forall x, \forall \phi \in \Phi, \, |a(x, \phi)| \leq \Tilde M := \max\ens{|\log(c)|, |\log(M)|}.
    \end{equation}
    
    While this is sufficient to apply Lemma 2.4 from \cite{NEWEY19942111}, this lemma only require to bound $a(x, \phi)$ by some quantity $d(x)$ that is independent of $\phi$ and integrable in $x$.
    
    We have 
    \begin{align}
        |a(x, \phi)| &= |\log q_\phi(x)|\\
        &= \left( \log q_\phi(x) \right)^{+} + \left( \log q_\phi(x) \right)^{-} \label{eq:val-abs-pos-neg-parts}
    \end{align}
    where $z^+$ and $z^-$ denote respectively the positive and negative parts of $z$.
    
    We have 
    \begin{align}
        \left( \log q_\phi(x) \right)^{+} &= \max(0, \log q_\phi(x)) \\
        &= \max\left(0, \log \int_{\xi} p_\xi(s' \mid s, a) p_\phi(\xi) \di \xi \right),
    \end{align}
    and by \Cref{A.2}, $p_\xi(s' \mid s, a) \leq M$, hence $\left( \log q_\phi(x) \right)^{+} \leq |\log(M)|$. Thus, the first term of \eqref{eq:val-abs-pos-neg-parts} is bounded by $|\log(M)|$ which is independent of $\phi$ and integrable in $x$.
    
    Furthermore,
    \begin{align}
        \left( \log q_\phi(x) \right)^{-} &= \max(0, - \log q_\phi(x))\\
        &= \max\left(0, \log \frac{1}{q_\phi(x)}\right) \\
        &\leq \max\left(0, \sup_{\phi} \log \frac{1}{q_\phi(x)} \right) \\
        &= \max\left(0,  \log \sup_{\phi}\frac{1}{q_\phi(x)} \right) \label{eq:log-sup} \\
        &= \max\left(0,  \log \frac{1}{\inf_{\phi} q_\phi(x)} \right),
    \end{align}
    where Equality \eqref{eq:log-sup} follows from \Cref{lem:sup-log}. The last quantity is independent of $\phi$, so we only need it to be integrable in order for the weak consistency result to hold.
    
    Since this quantity is non-negative, \Cref{lem:esperance-positive} yields
    \begin{align}
        \E\left[ \max\left(0,  \log \frac{1}{\inf_{\phi} q_\phi(x)} \right) \right] &= \int_{0}^\infty P\left( \max\left(0,  \log \frac{1}{\inf_{\phi} q_\phi(x)} \right) \geq \alpha \right)  \di \alpha \\
        &= \int_{0}^\infty P\left( \log \frac{1}{\inf_{\phi} q_\phi(x)}  \geq \alpha \right)  \di \alpha \\
        &=  \int_{0}^\infty P\left( \inf_{\phi} q_\phi(x)  \leq e^{-\alpha} \right)  \di \alpha,
    \end{align}
    
    and hence we only need to have the convergence of this integral. The integrand is bounded (between $0$ and $1$), so the integral is always convergent on $(0, 1]$. Hence, it is sufficient to have the convergence of the integral on $[1, \infty)$, e.g., one sufficient condition might be 
    \begin{equation}
        P\left( \inf_{\phi} q_\phi(x)  \leq e^{-\alpha} \right)  \leq \frac{1}{\alpha^2} \text{ for } \alpha \text{ sufficiently large},
    \end{equation}
    equivalently,
    \begin{equation}
        P\left( \inf_{\phi} q_\phi(x)  \leq \epsilon \right) \leq \frac{1}{(\log(\epsilon))^2 } \text{ for } \epsilon \text{ sufficiently small}.
    \end{equation}
    
    Notice that \Cref{A.4} implies this condition, since it implies that $\inf_{\phi} q_\phi(x) > 0$ and hence for sufficiently small $\epsilon > 0$ we have
    \begin{equation}
        P\left( \inf_{\phi} q_\phi(x) \leq \epsilon \right) = 0 \leq \frac{1}{(\log(\epsilon))^2}. 
    \end{equation}
\end{proof}

\subsection{Sufficient Condition for the Uniform Lipschitz Continuity Assumption}\label{app:cs-lip}
In this section, we prove a practical sufficient condition for \Cref{A.6}. More formally, the following holds:

\begin{restate-lem}{\ref{lem:CS-lip}}[Sufficient Condition for the Uniform Lipschitz Continuity Assumption]
    Suppose the following holds for every $x = (s, a, s')$
    \begin{enumerate}
        \item The function $\xi \mapsto p_\xi(s' \mid s, a)$ is twice continuously differentiable (of class $C^2$),
        \item There exists two constants $G_1 > 0$ and $G_2 > 0$ such that $\left|\nabla_\xi p_\xi(s' \mid s, a)  \right| \leq G_1$ and $|\nabla_\xi^2 p_\xi(s' \mid s, a)  | \leq G_2$ ,
    \end{enumerate}
    then \Cref{A.6} holds with $L = \dfrac{G_1 + G_2 / 2}{c}$.
\end{restate-lem}

Before proving this result, state and prove a technical lemma that we use in our proof.
\begin{lem}\label{lem:log-lip}
    For any $c > 0$, the logarithm function $\log$ is $\frac{1}{c}$-Lipschitz on $[c, \infty)$.
\end{lem}
\begin{proof}[Proof of \Cref{lem:log-lip}]
    Let $x$ and $y$ be two real numbers such that $c \leq x < y$.
    We have
    \begin{align}
        \left| \log(y) - \log(x) \right| &= \log(y) -\log(x) =\log\left( \dfrac{y}{x} \right) = \log\left(1 + \dfrac{y - x}{x} \right)\leq \dfrac{y-x}{x},
    \end{align}
    and since $x \geq c$, it follows
    \begin{equation}
        \left| \log(y) - \log(x) \right| \leq \frac{1}{c} (y - x) = \frac{1}{c} | y - x |.
    \end{equation}
\end{proof}
Notice that this result can also be proved using the \emph{mean value inequality}. 

\begin{proof}[Proof of \Cref{lem:CS-lip}]
    Our goal is to prove that under the two assumptions of \Cref{lem:CS-lip}, we have
  \begin{equation}
      \forall \phi := (\mu, \Sigma),\phi' := (\mu', \Sigma') \in\Phi, \forall x ,
    \bigl|a(x,\phi)-a(x,\phi')\bigr|
    \le L\,\bigl\|\phi-\phi'\bigr\|_2.
  \end{equation}

  First, notice that using \Cref{lem:log-lip} and \Cref{A.4}, we have
  \begin{equation}
      \left| a(x, \phi) - a(x, \phi') \right| = |\log(f_x(\phi)) - \log(f_x(\phi'))| \leq \frac{1}{c} | f_x(\phi) - f_x(\phi') |, 
  \end{equation}
  where we used the notation $f_x(\phi) := q_\phi(s' \mid s, a) = \E_{\xi \sim P_\phi}[p_\xi(s' \mid s, a)]$. Hence, it is sufficient to prove that $|f_x(\phi) - f_x(\phi') | \leq \Tilde L \norm{\phi - \phi'}$ for every $x$ for some constant $\Tilde L > 0$.
  
  We start by treating the case where $\Sigma$ and $\Sigma'$ are non-singular.

  \paragraph{Case 1: non-singular covariance matrices.} 
    In the case where $\Sigma$ is non-singular, \begin{equation}
        f_x(\phi) = \int_{\xi} h_x(\xi) \mathcal N(\xi; \mu, \Sigma) \rm d \xi,
    \end{equation} where $h_x(\xi) := p_\xi(s' \mid s, a)$ and $\mathcal N (\xi; \mu, \Sigma) := (2\pi)^{-d/2}\det(\Sigma)^{-\tfrac12} \exp\Bigl( - \tfrac12(\xi - \mu)^\top \Sigma^{-1} (\xi - \mu) \Bigr)$.
    Since $\mu \mapsto \mu^\top$ and $\Sigma \mapsto \Sigma^{-1}$ are continuously differentiable respectively on $\R^d$ and $\mathrm{GL}_d(\R)$, then the function $\phi \mapsto \mathcal{N}(\xi; \mu, \Sigma)$ is $C^1$ as long as $\Sigma \succ 0$ with 
    \begin{equation}
        \boxed{\nabla_\mu \mathcal{N}(\xi; \mu, \Sigma) = \Sigma^{-1} (\xi - \mu) \, \mathcal{N}  (\xi; \mu, \Sigma).}
    \end{equation}
    and using the matrix‐calculus identities
    \begin{equation}
        \di \log \det \Sigma = \tr(\Sigma^{-1} \di \Sigma), \qquad \di(\Sigma^{-1}) = - \Sigma^{-1}(\di \Sigma) \Sigma^{-1},
    \end{equation}
    we compute
    \begin{align}
    \di \log \mathcal N
    &= \di \Bigl[-\tfrac12 \log \det \Sigma -\tfrac12 (\xi - \mu)^\top \Sigma^{-1} (\xi-\mu) \Bigr]\\
    &= - \tfrac12 \tr(\Sigma^{-1} \di \Sigma) - \tfrac12 (\xi - \mu)^\top \di (\Sigma^{-1})(\xi - \mu)\\
    &= - \tfrac12 \tr( \Sigma^{-1} \di \Sigma) + \tfrac12 (\xi - \mu)^\top \bigl[ \Sigma^{-1}(\di \Sigma ) \Sigma^{-1} \bigr]
           (\xi-\mu).
    \end{align}
    Since $\di \mathcal N = \mathcal N \di \log \mathcal N$, we get
    \begin{equation}
         \di \mathcal N = \tfrac12 \mathcal N \Bigl[( \xi -\mu )^\top \Sigma^{-1} (\di \Sigma) \Sigma^{-1} (\xi - \mu) - \tr(\Sigma^{-1} \di \Sigma) \Bigr].
    \end{equation}
    Rewriting in Frobenius inner product form,
    \begin{equation}
        \di \mathcal N = \tr \Bigl[\Bigl(\tfrac12 \mathcal N [\Sigma^{-1} (\xi - \mu) (\xi - \mu)^\top \Sigma^{-1} - \Sigma^{-1}] \Bigr)^\top \di \Sigma \Bigr].
    \end{equation}

    Thus the gradient is
    \begin{equation}
        \boxed{ \nabla_\Sigma \mathcal N(\xi; \mu, \Sigma) = \frac12 \mathcal N(\xi;\mu,\Sigma) \Bigl[ \Sigma^{-1} (\xi - \mu)(\xi - \mu)^\top \Sigma^{-1} - \Sigma^{-1} \Bigr].}
    \end{equation}
    
    On each compact subset $K$ of $\Phi \, \cap \, \ens{(\xi, \Sigma): \Sigma \succ 0}$, we have by the sub-multiplicativity of the norm
    \begin{equation}
        \norm{h_x(\xi) \nabla_\mu \mathcal N(\xi; \mu, \Sigma)}_2 \leq M \norm{\Sigma^{-1}}_2 \norm{\xi - \mu}_2 \mathcal N (\xi; \mu, \Sigma),
    \end{equation}
    ans since the function $\phi \mapsto \norm{\Sigma^{-1}}_2 \norm{\xi - \mu}_2 \mathcal N (\xi; \mu, \Sigma)$ is continuous on $K$, it attains its maximum in some point of $K$, hence, there exists some $\mu_0$ and $\Sigma_0 \succ 0$ such that for all $\phi \in K$,
    \begin{equation}
        \norm{h_x(\xi) \nabla_\mu \mathcal N(\xi; \mu, \Sigma)} \leq M \norm{\Sigma_0^{-1}} \norm{\xi - \mu_0} \mathcal N (\xi; \mu_0, \Sigma_0), 
    \end{equation}
    where the right term is integrable in $\xi$ since $\E_{X\sim \mathcal N (\xi; \mu_0, \Sigma_0)} [\norm{X - \mu_0}] < \infty$.
    Furthermore,
    \begin{equation}
        \norm{h_x(\xi) \nabla_\Sigma \mathcal N(\xi; \mu, \Sigma)}_F \leq \frac{1}{2} M \mathcal{N}(\xi; \mu, \Sigma) \left( \norm{\Sigma^{-1} (\xi - \mu) (\xi - \mu)^\top \Sigma^{-1}}_F + \norm{\Sigma^{-1}}_F \right),
    \end{equation}
    and $\norm{\Sigma^{-1} (\xi - \mu) (\xi - \mu)^\top \Sigma^{-1}}_F \leq \norm{\Sigma^{-1}}_F \norm{(\xi - \mu) (\xi - \mu)^\top}_F \norm{\Sigma^{-1}}_F$. The middle factor can be rewritten as follows
    \begin{align}
        \norm{(\xi - \mu) (\xi - \mu)^\top}_F &= \tr\left[(\xi - \mu) (\xi - \mu)^\top (\xi - \mu) (\xi - \mu)^\top \right]\\
        &= \tr\left[(\xi - \mu)^\top  (\xi - \mu) (\xi - \mu)^\top (\xi - \mu) \right] \\
        &= \norm{\xi - \mu}_2^2,
    \end{align}
    which yields
    \begin{equation}
        \norm{h_x(\xi) \nabla_\Sigma \mathcal N(\xi; \mu, \Sigma)}_F \leq \frac{1}{2} M \mathcal{N}(\xi; \mu, \Sigma) \left( \norm{\Sigma^{-1}}_F^2 \norm{\xi - \mu}_2^2 + \norm{\Sigma^{-1}}_F\right).
    \end{equation}
    Again, the function $\phi \mapsto \left( \norm{\Sigma^{-1}}_F^2 \norm{\xi - \mu}_2^2 + \norm{\Sigma^{-1}}_F \right) \mathcal N (\xi; \mu, \Sigma)$ is continuous on $K$, it attains its maximum in some point of $K$, hence, there exists some $\mu_1$ and $\Sigma_1 \succ 0$ such that for all $\phi \in K$,
    \begin{equation}
        \norm{h_x(\xi) \nabla_\Sigma \mathcal N(\xi; \mu, \Sigma)}_F \leq \frac{1}{2} M \left( \norm{\Sigma_1^{-1}}_F^2 \norm{\xi - \mu_1}_2^2 + \norm{\Sigma_1^{-1}}_F \right) \mathcal N (\xi; \mu_1, \Sigma_1),
    \end{equation}
    where the right term in integrable in $\xi$ since the Gaussian distribution has finite second order moment.

    Using \emph{Leibniz integral rule}, the function $\phi \mapsto f_x(\phi)$ is $C^1$ and we may interchange differentiation and integration to get
    \begin{align}
        \nabla_\mu f_x(\phi) &= \int_\xi h_x(\xi) \nabla_\mu \mathcal N(\xi; \mu, \Sigma) \di \xi \\
        &= \int_\xi h_x(\xi) \Sigma^{-1} (\xi - \mu) \mathcal N(\xi; \mu, \Sigma) \di \xi \\
        &=\int_\xi h_x(\xi)  \left[ - \nabla_\xi \mathcal{N}(\xi; \mu, \Sigma)\right] \di \xi \\
        &= \int_\xi \nabla_\xi h_x(\xi) \, \mathcal{N}(\xi; \mu, \Sigma) \di \xi     \label{eq:ipp} \\
        &= \E_{\xi \sim \mathcal{N}(\mu, \Sigma)} \left[ \nabla_\xi h_x(\xi) \right],
    \end{align}
    where \Cref{eq:ipp} follows from an \emph{integration by part}.\footnote{The first term of the integration by part vanishes since $| h_x(\xi) \, \mathcal{N}(\xi; \mu, \Sigma)| \leq M \mathcal{N}(\xi; \mu, \Sigma) \xrightarrow[\norm{\xi} \to \infty]{} 0$.}
    Furthermore,
    \begin{align}
        \nabla_\Sigma f_x(\phi) &= \int_\xi h_x(\xi) \nabla_\Sigma \mathcal N(\xi; \mu, \Sigma) \di \xi \\
        &= \int_\xi h_x(\xi) \frac12  \Bigl[ \Sigma^{-1} (\xi - \mu)(\xi - \mu)^\top \Sigma^{-1} - \Sigma^{-1} \Bigr] \mathcal N(\xi;\mu,\Sigma) \di \xi \\
        &= \frac{1}{2} \Sigma^{-1} \E_{\xi \sim \mathcal N(\mu, \Sigma)}\Bigl[ h_x(\xi) \left[(\xi - \mu) (\xi - \mu)^\top - \Sigma \right] \Bigr] \Sigma^{-1} \\
        &= \frac{1}{2} \Sigma^{-1/2} \E_{\xi \sim \mathcal N(\mu, \Sigma)}\Bigl[ h_x(\xi) \left[\Sigma^{-1/2}(\xi - \mu) (\Sigma^{-1/2}(\xi - \mu))^\top - \bm{I}_d \right] \Bigr] \Sigma^{-1/2} \\
        &= \frac{1}{2} \Sigma^{-1/2} \E_{\xi \sim \mathcal N(\mu, \Sigma)}\left[g(\bm z) (\bm z \bm z^\top - \mathbf{I}_d)\right]   \Sigma^{-1/2}, \label{eq:54}
    \end{align}
    where $\Sigma^{-1/2}$ is the unique positive definite square root of $\Sigma^{-1}$, $\bm z := \Sigma^{-1/2}(\xi - \mu)$ and $g( \bm z) := h_x(\xi) = h_x(\Sigma^{1/2} \bm z + \mu)$. Using the \emph{Iterated Stein formula} \citep{stein2020, stein} we have
    \begin{align}
        \E_{\xi \sim \mathcal N(\mu, \Sigma)}\left[g(\bm z) (\bm z \bm z^\top - \mathbf{I}_d)\right] &= \E_{\xi \sim \mathcal N(\mu, \Sigma)}\left[\nabla_{\bold{z}}^2 g(\bm z)\right]\\ & = \E_{\xi \sim \mathcal N(\mu, \Sigma)}\left[\Sigma \, \nabla^2_\xi h_x(\xi)\right] \\ &= \Sigma \,  \E_{\xi \sim \mathcal N(\mu, \Sigma)}\left[\nabla^2_\xi h_x(\xi)\right].
    \end{align}
    Combining this equation with \Cref{eq:54} yields
    \begin{align}
        \nabla_\Sigma f_x(\phi) = \frac{1}{2} \E_{\xi \sim \mathcal N(\mu, \Sigma)}\left[ \nabla^2_\xi h_x(\xi) \right].
    \end{align}
    Since $f_x$ is $C^1$ when $\Sigma \succ 0$, for any two points $\phi, \phi' \in \Phi$ such that $\Sigma \succ 0$ and $\Sigma' \succ 0$, there is $\Tilde \phi$ on the segment joining them (and thus $\Tilde \Sigma \succ 0$) \footnote{Indeed, there exists $t \in [0,1]$ such that $\Tilde \Sigma = t \Sigma + (1-t) \Sigma'$ where $\Sigma \succ 0$ and $\Sigma' \succ 0$, thus for any $z \in \R^d, z^\top \Tilde \Sigma z = t z^\top \Sigma z + (1 - t) z^\top \Sigma' z > 0$.} so that by the \emph{mean‐value theorem}
    \begin{equation}
        f_x(\phi)-f_x(\phi') = \bigl \langle \nabla_\phi f_x(\tilde\phi),\,\phi-\phi' \bigr\rangle.
    \end{equation}

    In particular
    \begin{equation}
        |f_x(\phi)-f_x(\phi')| \le \|\nabla_\phi f_x(\tilde\phi)\|\,\|\phi-\phi'\|.
    \end{equation}

    By assumption (ii), $\|\nabla_\xi h_x\|\le G_1$ and $\|\nabla^2_\xi h_x\|\le G_2$. Hence
    \begin{equation}
          \|\nabla_\mu f_x(\phi)\| = \bigl\| \E[ \nabla_\xi h_x(\xi) ]\bigr\| \leq G_1, \quad \| \nabla_\Sigma f_x(\phi) \| = \tfrac12 \bigl\| \E[ \nabla^2_\xi h_x(\xi) ] \bigr\| \leq \tfrac{G_2}{2}.
    \end{equation}
    Assembling the two blocks,
    \begin{equation}
        |f_x(\phi) - f_x(\phi')| \leq \bigl(G_1 + \tfrac{G_2}{2}\bigr) \| \phi-\phi' \|.
    \end{equation}
    Therefore $f_x$ is Lipschitz in $\phi$, with constant $L'=G_1+G_2/2$, and by Lemma~\ref{lem:log-lip} so is  $a(x,\phi)=\log f_x(\phi)$ with constant $L = \dfrac{G_1 + G_2 / 2}{c}$.

    \paragraph{General case.} For the case where we no longer suppose that $\Sigma$ and $\Sigma'$ are non-singular, we use the density of the set of invertible matrices in $M_d(\R)$. More precisely, there exists two sequences of non-singular matrices $\ens{\Sigma_N}_N$ and $\ens{\Sigma'_N}_N$ such that $\Sigma_N \to \Sigma$ and $\Sigma'_N \to \Sigma'$ when $N \to \infty$. We denote $\phi_N := (\mu, \Sigma_N)$ and $\phi'_N := (\mu, \Sigma_N^\prime)$. The previous result yields 
    \begin{equation}
        \forall N \geq 0, \forall x, \, |a(x, \phi_N) - a(x, \phi'_N)| \leq L \norm{\phi_N - \phi'_N},
    \end{equation}
    and thus, when $N \to \infty$ we get
    \begin{equation}
        \forall x, \, |a(x, \phi) - a(x, \phi')| \leq L \norm{\phi - \phi'},
    \end{equation}
    where we used the continuity of the function $\phi \mapsto a(x, \phi)$ on $\Phi$ (\Cref{lemma:continu}). This concludes the proof.
\end{proof}

Notice that even if in many robotic systems have strongly non-smooth dynamics in $(s, a)$ due to hard contacts and friction. However, this non-smoothness concerns the map $(s, a) \mapsto p_\xi (s' \mid s, a)$ (for example, if the state encodes the position and velocity of a robot arm, near a hard contact the probability of next-step positive velocity can change discontinuously). However, the map $\xi \mapsto p_\xi ( s' \mid s, a)$ typically remains smooth with respect to the physical parameters. Intuitively, if we slightly perturb masses, friction coefficients, or gains, we expect the transition probabilities to change only slightly, even though the contact dynamics in $(s, a)$ are themselves non-smooth.

\subsection{Weak Consistency under Partial Coverage}\label{apx:partial-coverage}

\begin{restate-thm}{\ref{thm:coverage-relaxed}}
    Under Assumptions ~\ref{A.2}, \ref{A.3} and \ref{A.4}, the following holds, Any measurable maximizer 
            \(\displaystyle \widehat{\phi}_N\in\argmax_{\phi\in\Phi}L_N(\phi)\) 
            satisfies $\mathrm{dist} (\widehat{\phi}_N, \mathcal{Q}_\mu^\star) \xrightarrow[N\xrightarrow{} \infty]{P} 0$ \footnote{where $\mathrm{dist}$ is the distance to a set defined by $\mathrm{dist}(\phi, \mathcal{Q}) := \inf_{\psi \in \mathcal{Q}} \norm{\phi - \psi}$.}.
\end{restate-thm}

\begin{proof}[Proof of \Cref{thm:coverage-relaxed}]
    As in \Cref{thm:consistency}, the uniform law of large numbers holds:
\begin{equation}
\label{eq:ULLN}
\sup_{\phi\in\Phi} \big| L_N(\phi)-L(\phi) \big| \xrightarrow{ P } 0.
\end{equation}
\Cref{lem:Q-nonempty-compact} proves that $\mathcal{Q}^\star_\mu$ is nonempty and compact.

Fix $\varepsilon>0$ and define the separation (margin) outside the $\varepsilon$-neighborhood of $\mathcal{Q}^\star_\mu$:
\[ \eta(\varepsilon)\ := \inf\Big\{ L(\phi^\star)-L(\phi) : \phi^\star \in \mathcal{Q}^\star_\mu,\ \mathrm{dist}(\phi, \mathcal{Q}^\star_\mu)\ge \varepsilon \Big\}.
\]
Because $L$ is continuous and $\{\phi\in\Phi:\mathrm{dist}(\phi,\mathcal{Q}^\star_\mu)\ge \varepsilon\}$ is compact, we have $\eta(\varepsilon)>0$.

By \eqref{eq:ULLN}, there exists a sequence of events $\mathcal E_N$ with $P(\mathcal E_N)\to1$ such that on $\mathcal E_N$,
\[
\sup_{\phi\in\Phi}\big|L_N(\phi)-L(\phi)\big|\ \le\ \tfrac13\,\eta(\varepsilon).
\]
On $\mathcal E_N$, for any $\phi$ with $\mathrm{dist}(\phi,\mathcal{Q}^\star_\mu)\ge \varepsilon$ and any $\phi^\star \in \mathcal{Q}^\star_\mu$,
\[
L_N(\phi)\ \le\ L(\phi)+\tfrac13\eta(\varepsilon)\ \le\ L(\phi^\star)-\eta(\varepsilon)+\tfrac13\eta(\varepsilon)
\ =\ L(\phi^\star)-\tfrac23\eta(\varepsilon)
\ < \sup_{\psi\in\mathcal{Q}^\star_\mu} L_N(\psi),
\]
where the last inequality uses $L_N(\psi)\ge L(\psi)-\tfrac13\eta(\varepsilon)=L(\phi^\star)-\tfrac13\eta(\varepsilon)$ for any $\psi \in \mathcal{Q}^\star_\mu$.
Therefore, no maximizer of $L_N$ can lie outside the $\varepsilon$-neighborhood of $\mathcal{Q}^\star_\mu$ on $\mathcal E_N$.
Equivalently,
\[
\mathrm{dist}\big(\widehat{\phi}_N,\mathcal{Q}^\star_\mu\big)\ <\ \varepsilon\quad\text{on }\mathcal E_N.
\]
Since $P(\mathcal E_N)\to1$ and $\varepsilon>0$ is arbitrary, we conclude
$\mathrm{dist}(\widehat{\phi}_N,\mathcal{Q}^\star_\mu)\xrightarrow{ P }0$.
\end{proof}

\begin{restate-lem}{\ref{lem:Q-nonempty-compact}}
    Under Assumptions \ref{A.2}, \ref{A.3} and \ref{A.4} The identified set $\mathcal{Q}_\mu^\star$ is non-empty and compact and and the correspondence $\mu \mapsto \mathcal{Q}^\star_\mu$ is upper hemicontinuous\footnote{A set-valued map $F$ is upper hemicontinuous at $x_0$ if, whenever $x_n\!\to x_0$ and $y_n\in F(x_n)$ with $y_n\to y$, then $y\in F(x_0)$. Equivalently: for every open $U$ with $F(x_0)\subseteq U$, there exists a neighborhood $V$ of $x_0$ such that $F(x)\subseteq U$ for all $x\in V$.} with respect to total variation.
\end{restate-lem}

\begin{proof}[Proof of Lemma~\ref{lem:Q-nonempty-compact}]
Write 
\[
L(\phi, \mu)=\E_{(S,A)\sim\mu}\,\E_{S'\sim p_{\xi^\star}(\cdot\mid S,A)}[\,a((S,A,S'),\phi)\,]
=\int_{\mathcal S\times\mathcal A} f_\phi(s,a)\,\mu(ds,da), \]
where
\[
f_\phi(s,a):=\E_{S'\mid s,a}[a((s,a,S'),\phi)].
\]

\paragraph{Step 1: Finite-valued and continuity in $\phi$.}
We have 
\[
\sup_{\phi\in\Phi}\,|a(x,\phi)|\ \le\ \widetilde M\qquad\text{for all }x=(s,a,s').
\]
Therefore $|f_\phi(s,a)|\le \widetilde M$ for all $(s,a)$ and $\phi$, and $L(\phi,\mu)\in\R$. Moreover, \Cref{lemma:continu} gives continuity of $\phi\mapsto a(x,\phi)$ for each fixed $x$. By dominated convergence with the uniform bound $\widetilde M$, we obtain continuity (hence upper semicontinuity) of $\phi \mapsto L(\phi, \mu)$ on $\Phi$.

\paragraph{Step 2: Uniform TV–continuity in $\mu$.}
Let $\mu_n\to\mu$ in total variation. Then, for any $\phi\in\Phi$,
\begin{align}
   |L(\phi,\mu_n)-L(\phi,\mu)| &=\Big|\int f_\phi(s,a)\,(\mu_n-\mu)(ds\,da)\Big| \\
&\ \le\ \int |f_\phi(s,a)|\,|(\mu_n-\mu)|(ds\,da) \\
&\ \le\ \widetilde M\,\|\mu_n-\mu\|_{\mathrm{TV}}. 
\end{align}
Taking the supremum over $\phi\in\Phi$ yields
\begin{equation}\label{eq:uniform-tv}
\sup_{\phi\in\Phi}|L(\phi,\mu_n)-L(\phi;\mu)|\ \le\ \widetilde M\,\|\mu_n-\mu\|_{\mathrm{TV}}\ \xrightarrow[n\to\infty]{}\ 0.
\end{equation}

\paragraph{Step 3: Joint continuity of $L$.}
Let $(\phi_n,\mu_n)\to(\phi,\mu)$ with $\phi_n\to\phi$ in $\Phi$ and $\mu_n\to\mu$ in TV. Then
    \[ 
    | L(\phi_n,\mu_n) - L(\phi,\mu) | \le | L(\phi_n,\mu_n) -L(\phi_n,\mu)| + |L(\phi_n,\mu) - L(\phi,\mu)|.
    \]
By uniform TV–continuity in $\mu$ (from $|a(x,\phi)| \le \widetilde M$),
    \[ \sup_{\psi\in\Phi}|L(\psi,\mu_n)-L(\psi,\mu)| \le \tilde M \| \mu_n - \mu \|_{\mathrm{TV}} \xrightarrow[n\to\infty]{}0, 
    \]
hence $|L(\phi_n,\mu_n)-L(\phi_n,\mu)|\to0$. By continuity in $\phi$ at fixed $\mu$ (dominated convergence with the same bound),
$|L(\phi_n,\mu)-L(\phi,\mu)|\to0$. Therefore $L(\phi_n,\mu_n)\to L(\phi,\mu)$, i.e., $(\phi,\mu)\mapsto L(\phi,\mu)$ is jointly continuous.

Hence, by \emph{Berge’s Maximum Theorem} \citep{berge1963topological}, for each $\mu$ the argmax set $\mathcal Q^\star_\mu=\arg\max_{\phi\in\Phi}L(\phi, \mu)$ is nonempty and compact, and the correspondence $\mu \mapsto \mathcal Q^\star_\mu$ is upper hemicontinuous (in total variation).
\end{proof}

\subsection{Misspecification and Representative MDPs}\label{apx:misspecification}
In this subsection, we study the case where we no longer suppose that the true dynamics $\M^\star$ belongs to the simulator class $\U$.
Given a Markov kernel $P : \mathcal S \times \mathcal A \to \Delta(\mathcal S)$,
we write $\mathcal M(P)$ for the MDP
$(\mathcal S,\mathcal A,P,R,H,s_1)$
and we denote the value of a policy $\pi$ at the initial state $s_1$ by
\[
    V_{P,1}^{\pi}(s_1) := V_{\mathcal M(P),1}^{\pi}(s_1).
\]
We use $\mathrm{TV}(p,q) := \frac12 \sum_{s' \in \mathcal S} |p(s')-q(s')|$ for the
total variation distance between distributions $p,q$ over~$\mathcal S$.
\begin{lem}[Value stability under kernel perturbations]\label{lem:value-stability}
Let $P,Q : \mathcal{S} \times \mathcal{A} \to \Delta(\mathcal{S})$ be two transition kernels defined on the same state--action space, with a common reward function $R : \mathcal{S} \times \mathcal{A} \to [0,1]$ and horizon $H$. Define
\[
    \Delta(P,Q) := \sup_{(s,a)\in\mathcal{S}\times\mathcal{A}}
    \mathrm{TV}\big(P(\cdot\mid s,a),\,Q(\cdot\mid s,a)\big),
\]
where $\mathrm{TV}(p,q) := \tfrac12\sum_{s' \in \mathcal{S}} |p(s') - q(s')|$ denotes the total variation distance
between two probability distributions $p,q$ on $\mathcal{S}$. Then, for any policy $\pi$ and initial state $s_1$,
\[
    \bigl| V_{P,1}^{\pi}(s_1) - V_{Q,1}^{\pi}(s_1) \bigr|
    \le H^2 \Delta(P,Q).
\]
\end{lem}
\begin{proof}
For $h \in \{1,\dots,H\}$, let $d_h^{P}$ and $d_h^{Q}$ denote the distributions over state--action pairs
$(s_h,a_h)$ at step $h$ when running policy $\pi$ in the MDPs $\mathcal{M}(P)$ and $\mathcal{M}(Q)$ respectively,
both initialized from the same state $s_1$. In particular, $d_1^{P} = d_1^{Q}$.
We first control the evolution of the occupancy measures. By definition of the dynamics,
\[
    d_{h+1}^{P}(s',a')
    = \sum_{s,a} d_h^{P}(s,a) \pi_h(a'\mid \mathrm{traj}_h) P(s'\mid s,a),
\]
and analogously for $Q$. Hence
\begin{align*}
    \|d_{h+1}^{P} - d_{h+1}^{Q}\|_1 
    &= \sum_{s', a'} \left| \sum_{s, a} d_h^{P}(s,a) \pi_h(a'\mid \mathrm{traj}_h) P(s'\mid s,a) - d_h^{Q}(s,a) \pi_h(a'\mid \mathrm{traj}_h) Q(s'\mid s,a) \right|  \\
    &\leq \sum_{s', a'} \sum_{s, a} \bigl|d_h^{P}(s,a) - d_h^{Q}(s,a)\bigr| \pi_h(a'\mid \mathrm{traj}_h) P(s'\mid s,a) \\
    & \qquad \qquad \qquad + \sum_{s', a'} \sum_{s, a} d_h^{Q}(s,a)  \pi_h(a'\mid \mathrm{traj}_h)  \bigl|P(s'\mid s,a) - Q(s'\mid s,a)\bigr| \\
    &\le \sum_{s,a} \bigl|d_h^{P}(s,a) - d_h^{Q}(s,a)\bigr|
      + \sup_{(s,a)} \sum_{s'} \bigl|P(s'\mid s,a) - Q(s'\mid s,a)\bigr| \\
    &= \|d_h^{P} - d_h^{Q}\|_1
      + 2 \sup_{(s,a)} \mathrm{TV}\big(P(\cdot\mid s,a),Q(\cdot\mid s,a)\big) \\
    &\le \|d_h^{P} - d_h^{Q}\|_1 + 2 \Delta(P,Q).
\end{align*}
Since $\|d_1^{P} - d_1^{Q}\|_1 = 0$, an induction on $h$ yields
\[
    \|d_h^{P} - d_h^{Q}\|_1 \;\le\; 2(h-1)\,\Delta(P,Q)
    \qquad \forall\, h=1,\dots,H.
\]
Next, write the value of policy $\pi$ under kernel $P$ as
\[
    V_{P,1}^{\pi}(s_1)
    = \sum_{h=1}^H \mathbb{E}_{(s_h,a_h)\sim d_h^{P}} \bigl[R(s_h,a_h)\bigr],
\]
and similarly $V_{Q,1}^{\pi}(s_1)$ with $d_h^{Q}$. Since $R(s,a)\in[0,1]$,
\[
    \bigl|\mathbb{E}_{d_h^{P}}[R] - \mathbb{E}_{d_h^{Q}}[R]\bigr|
    \le \|d_h^{P} - d_h^{Q}\|_1.
\]
Therefore,
\begin{align*}
    \bigl|V_{P,1}^{\pi}(s_1) - V_{Q,1}^{\pi}(s_1)\bigr|
    &\le \sum_{h=1}^H \bigl|\mathbb{E}_{d_h^{P}}[R] - \mathbb{E}_{d_h^{Q}}[R]\bigr| \\
    &\le \sum_{h=1}^H \|d_h^{P} - d_h^{Q}\|_1 \\
    &\le \sum_{h=1}^H 2(h-1)\,\Delta(P,Q)
     \le H^2 \Delta(P,Q),
\end{align*}
where the last inequality uses $\sum_{h=1}^H (h-1) = H(H-1)/2 \le H^2/2$. This concludes the proof.
\end{proof}
\begin{thm}
Let $q_{\widehat \phi_N}$ be the learned mixture kernel from offline data, $\M_{\widehat \phi_N} := \M(q_{\widehat \phi_N})$ be the training MDP ODR uses, and $\pi_N$ be the learned policy using any RL algorithm in $\M_{\widehat \phi_N}$, then we have:
    \[
    \mathrm{Gap}_{\mathcal M^\star}(\pi_N) \le
    \mathrm{Gap}_{\mathcal M_{\widehat\phi_N}}(\pi_N) + 
    4 H^2 \Delta\bigl(P^\star,q_{\widehat\phi_N}\bigr).
    \]
\end{thm}
\begin{proof}
We have 
\begin{align*}
    &V_{\M^*, 1}^*(s_1) - V_{\M^*, 1}^{\pi_N}(s_1)\\ & = V_{\M^*, 1}^*(s_1) - V_{\M_{\widehat \phi_N}, 1}^*(s_1) + V_{\M_{\widehat \phi_N}, 1}^*(s_1) - V_{\M_{\widehat \phi_N}, 1}^{\pi_N}(s_1) + V_{\M_{\widehat \phi_N}, 1}^{\pi_N}(s_1)  - V_{\M^*, 1}^{\pi_N}(s_1),
\end{align*}
By maximality of $V_{\M_{\widehat \phi_N}, 1}^*(s_1)$ we have:
    \begin{align}
        V_{\M^*, 1}^*(s_1) - V_{\M_{\widehat \phi_N}, 1}^*(s_1) &\leq V_{\M^*, 1}^*(s_1) - V_{\M_{\widehat \phi_N}, 1}^{\pi^*_{\M*}}(s_1) \\
        &= V_{\M^*, 1}^{\pi^*_{\M*}}(s_1) - V_{\M_{\widehat \phi_N}, 1}^{\pi^*_{\M*}}(s_1) \\
        &\leq H^2 \Delta(P^*, q_{\widehat \phi_N}),
    \end{align}
where the last inequality follows from \Cref{lem:value-stability}. Using the same lemma we have:
    \[V_{\M_{\widehat \phi_N}, 1}^{\pi_N}(s_1)  - V_{\M^*, 1}^{\pi_N}(s_1) \leq H^2 \Delta(P^*, q_{\widehat \phi_N}). \]
Hence,
    \[V_{\M^*, 1}^*(s_1) - V_{\M^*, 1}^{\pi_N}(s_1) \leq V_{\M_{\widehat \phi_N}, 1}^*(s_1) - V_{\M_{\widehat \phi_N}, 1}^{\pi_N}(s_1) + 2 H^2 \Delta(P^*, q_{\widehat \phi_N}). \]
\end{proof}
The first term is the suboptimality of $\pi_N$ in the \emph{learned} mixture MDP
(that ODR actually optimizes against), while the second term is a
\emph{closeness penalty} measuring how well the fitted mixture kernel
$q_{\widehat\phi_N}$ approximates the real dynamics $P^\star$ in total variation.
Under the well-specified and identifiable assumptions of our main results,
$q_{\widehat\phi_N}$ converges to $P^\star$, so the penalty term vanishes as
$N\to\infty$. Under misspecification, the penalty converges to the best approximation error achievable within the simulator family.

\end{document}

%% file: math_commands.tex
\usepackage{amsmath,amsfonts,bm}

\def\eqref#1{equation~\ref{#1}}

\def\1{\bm{1}}

\DeclareMathAlphabet{\mathsfit}{\encodingdefault}{\sfdefault}{m}{sl}
\SetMathAlphabet{\mathsfit}{bold}{\encodingdefault}{\sfdefault}{bx}{n}

\newcommand{\E}{\mathbb{E}}

\newcommand{\R}{\mathbb{R}}

\DeclareMathOperator*{\argmax}{arg\,max}

\newcommand{\M}{\mathcal{M}}
\newcommand{\U}{\mathcal{U}}
\newcommand{\D}{\mathcal{D}}

\newcommand*{\N}{\mathbb{N}}

\newcommand*{\ens}[1]{\left\{#1\right\}}%
\newcommand*{\enstq}[2]{\left\{#1\,:\,#2\right\}}%

\newcommand\norm[1]{\left\lVert#1\right\rVert}

\newcommand*{\abs}[1]{\left\lvert #1 \right\rvert}
